\documentclass[letterpaper,twocolumn,10pt]{article}
\usepackage{usenix}
\usepackage[available,functional,reproduced]{usenixbadges}

\usepackage[english]{babel}
\usepackage{blindtext}
\usepackage{makecell}
\usepackage{tabularx}
\usepackage{booktabs}
\usepackage{tikz}
\usepackage{amsmath}
\usepackage{amsthm}
\usepackage{algorithm}
\usepackage{algorithmic}
\usepackage{mathtools}
\usepackage{filecontents}
\usepackage{xspace}
\usepackage{multirow,multicol}
\usepackage{multicol}
\usepackage{ifthen}
\usepackage{amssymb}
\usepackage{optidef}
\usepackage{wrapfig}
\usepackage{graphicx}
\usepackage{subcaption}
\usepackage{float}
\usepackage[compact]{titlesec}

\def\Snospace~{\S{}}

\mathchardef\mhyphen="2D

\newcommand{\sys}{Alpa\xspace}
\newcommand{\aref}[2]{\hyperref[#1]{\ref*{#1}{#2}}}
\newtheorem{theorem}{Theorem}

\newcommand{\showcomments}{yes}

\newcommand\todo[1]{\ifthenelse{\equal{\showcomments}{yes}}{{\color{red} TODO: #1}}{\ignorespaces}}
\newcommand\joey[1]{\ifthenelse{\equal{\showcomments}{yes}}{{\color{blue} (Joey: #1)}}{\ignorespaces}}
\newcommand\lianmin[1]{\ifthenelse{\equal{\showcomments}{yes}}{{\color{blue} Lianmin: #1}}{\ignorespaces}}
\newcommand\Danyang[1]{\ifthenelse{\equal{\showcomments}{yes}}{{\color{red} [Danyang: #1]}}{\ignorespaces}}
\newcommand\ion[1]{\ifthenelse{\equal{\showcomments}{yes}}{{\color{blue} Ion: #1}}{\ignorespaces}}
\newcommand\yida[1]{\ifthenelse{\equal{\showcomments}{yes}}{{\color{blue} [Yida: #1]}}{\ignorespaces}}
\newcommand\cody[1]{\ifthenelse{\equal{\showcomments}{yes}}{{\color{blue} [Cody: #1]}}{\ignorespaces}}
\newcommand\hao[1]{\ifthenelse{\equal{\showcomments}{yes}}{{\color{blue} Hao: #1}}{\ignorespaces}}
\newcommand\zhuohan[1]{\ifthenelse{\equal{\showcomments}{yes}}{{\color{blue} Zhuohan: #1}}{\ignorespaces}}

\begin{document}
\date{}

\title{\Large \bf \sys: Automating Inter- and Intra-Operator Parallelism \\ for Distributed Deep Learning}

\author{
\rm{Lianmin Zheng$^{\text{1}, *}$ \enskip
    Zhuohan Li$^{\text{1}, *}$ \enskip
    Hao Zhang$^{\text{1}, *}$ \enskip
    Yonghao Zhuang$^{\text{4}}$ \enskip}\\
\rm{Zhifeng Chen$^{\text{3}}$ \enskip
    Yanping Huang$^{\text{3}}$ \enskip
    Yida Wang$^{\text{2}}$ \enskip
    Yuanzhong Xu$^{\text{3}}$ \enskip
    Danyang Zhuo$^{\text{6}}$ \enskip}\\
   \rm{Eric P. Xing$^{\text{5}}$ \enskip
Joseph E. Gonzalez$^{\text{1}}$ \enskip
    Ion Stoica$^{\text{1}}$ \enskip}\\
\\
  {$^{\text{1}}$UC Berkeley\enskip $^{\text{2}}$Amazon Web Services\enskip $^{\text{3}}$Google\enskip  $^{\text{4}}$Shanghai Jiao Tong University} \\
  {$^{\text{5}}$MBZUAI, Carnegie Mellon University\enskip $^{\text{6}}$Duke University}
}
\maketitle 

\begin{abstract}
\sys automates model-parallel training of large deep learning (DL) models by generating execution plans that unify data, operator, and pipeline parallelism.
Existing model-parallel training systems either require users to manually create a parallelization plan or automatically generate one from a limited space of model parallelism configurations. They do not suffice to scale out complex DL models on distributed compute devices.
\sys distributes the training of large DL models by viewing parallelisms as two hierarchical levels: inter-operator and intra-operator parallelisms.
Based on it, \sys constructs a new hierarchical space for massive model-parallel execution plans.
\sys designs a number of compilation passes to automatically derive efficient parallel execution plans at each parallelism level. \sys implements an efficient runtime to orchestrate the two-level parallel execution on distributed compute devices.
Our evaluation shows \sys generates parallelization plans that match or outperform hand-tuned model-parallel training systems even on models they are designed for.
Unlike specialized systems, \sys also generalizes to models with heterogeneous architectures and models without manually-designed plans. \sys's source code is publicly available at \url{https://github.com/alpa-projects/alpa}.
\end{abstract}

{\let\thefootnote\relax\footnote{{$^*$Lianmin, Zhuohan, and Hao contributed equally. Part of the work was done when Lianmin interned at Amazon and Zhuohan interned at Google.}}}
\section{Introduction}
\label{sec:intro}

Several of the recent advances~\cite{huang2019gpipe, brown2020language,shoeybi2019megatron} in deep learning (DL) have been a direct result of significant increases in model size. 
For example, scaling language models, such as GPT-3, to hundreds of billions of parameters~\cite{brown2020language} and training on much larger datasets enabled fundamentally new capabilities. 

However, training these extremely large models on distributed clusters currently requires a significant amount of engineering effort that is specific to both the model definition and the cluster environment.  
For example, training a large transformer-based language model requires heavy tuning and careful selection of multiple parallelism dimensions~\cite{narayanan2021efficient}. Training the large Mixture-of-Expert (MoE) transformers model~\cite{lepikhin2020gshard, du2021glam} on TPU clusters requires manually tuning the partitioning axis for each layer, whereas training the same model on an AWS GPU cluster calls for new pipeline schemes that can depend on the choices of partitioning (\S\ref{subsec:e2e}).

More generally, efficient large-scale model training requires tuning a complex combination of data, operator, and pipeline parallelization approaches at the granularity of the individual tensor operators.
Correctly tuning the parallelization strategy has been shown~\cite{li2021terapipe,lee2019automating} to deliver an order of magnitude improvements in training performance, but depends on strong machine learning (ML) and system expertise.

Automating the parallelization of large-scale models would significantly accelerate ML research and production by enabling model developers to quickly explore new model designs without regard for the underlying system challenges.
Unfortunately, it requires navigating a complex space of plans that grows exponentially with the dimensions of parallelism and the size of the model and cluster. 
For example, when all parallelism techniques are enabled, figuring out the execution plan involves answering a web of interdependent questions, such as how many data-parallel replicas to create, which axis to partition each operator along, how to split the model into pipeline stages, and how to map devices to the resulting parallel executables. 
The interplay of different parallelization methods and their strong dependence on model and cluster setups form a combinatorial space of plans to optimize.
Recent efforts~\cite{narayanan2019pipedream,fan2021dapple,wang2019supporting} to automatically parallelize model training are constrained to the space of a single model-parallelism approach, or rely on strong assumptions on the model and cluster specifications (\S\ref{sec:conventional}). 

Our key observation is that we can organize different parallelization techniques into a \emph{hierarchical space} and map these parallelization techniques to the \emph{hierarchical structure} of the compute cluster.
Different parallelization techniques have different bandwidth requirements for communication, while a typical compute cluster has a corresponding structure: closely located devices can communicate with high bandwidth while distant devices have limited communication bandwidth.

With this observation in mind, in this paper, we take a different view from conventional data and model parallelisms, and re-categorize ML parallelization approaches as \emph{intra-operator} and \emph{inter-operator} parallelisms. 
Intra-operator parallelism partitions ML operators along one or more tensor axes (batch or non-batch) and dispatches the partitions to distributed devices (Fig.~\aref{fig:all-plans}{c}); inter-operator parallelism, on the other hand, slices the model into disjoint stages and pipelines the execution of stages on different sets of devices (Fig.~\aref{fig:all-plans}{d}). They take place at two different granularities of the model computation, differentiated by whether to partition operators.

Given that, a parallel execution plan can be expressed \emph{hierarchically} by specifying the plan in each parallelism category, leading to a number of advantages. First, intra- and inter-operator parallelisms feature distinct characteristics: intra-operator parallelism has better device utilization, but results in communicating at every split and merge of partitioned operators, per training iteration; whereas inter-operator parallelism only communicates between adjacent stages, which can be light if sliced properly, but incurs device idle time due to scheduling constraints. We can harness the asymmetric nature of communication bandwidth in a compute cluster, and map intra-operator parallelism to devices connected with high communication bandwidth, while orchestrating the inter-operator parallelism between distant devices with relatively lower bandwidth in between.
Second, this hierarchical design allows us to solve each level near-optimally as an individual tractable sub-problem.
While the joint execution plan is not guaranteed globally optimal, they demonstrate strong performance empirically for training various large models.

Guided by this new problem formulation, we design and implement \sys, the first compiler that automatically generates parallel execution plans covering all data, operator, and pipeline parallelisms. Given the model description and a cluster configuration, \sys achieves this by partitioning the cluster into a number of \emph{device meshes}, each of which contains devices with preferably high-bandwidth connections,
and partitioning the computation graph of the model into \emph{stages}. It assigns stages to device meshes, and automatically orchestrates intra-operator parallelisms on a device mesh and inter-operator parallelisms between device meshes.

In summary, we make the following contributions:

\noindent \(\bullet\) We construct a two-level parallel execution plan space (Fig.~\aref{fig:all-plans}{e}) 
where plans are specified hierarchically using inter- and intra-operator parallelisms.

\noindent \(\bullet\) We design tractable optimization algorithms to derive near-optimal execution plans at each level.

\noindent \(\bullet\) We implement \sys, a compiler system  for distributed DL on GPU clusters. \sys features: (1) a set of compilation passes that generate execution plans using the hierarchical optimization algorithms, (2) a new runtime architecture that orchestrates the inter-op parallelism between device meshes, and (3) a number of system optimizations that improve compilation and address cross-mesh communication.

\noindent \(\bullet\) We evaluate \sys on training large models with billions of parameters. We compare \sys with state-of-the-art distributed training systems on an Amazon EC2 cluster of 8 p3.16xlarge instances with 64 GPUs.
On GPT~\cite{brown2020language} models, \sys can match 
the specialized system Megatron-LM~\cite{shoeybi2019megatron, narayanan2021efficient}.
On GShard MoE models~\cite{lepikhin2020gshard}, compared to a hand-tuned system Deepspeed\cite{rasley2020deepspeed}, \sys achieves a $3.5\times$ speedup on 2 nodes and a $9.7\times$ speedup on 4 nodes.
Unlike specialized systems, \sys also generalizes to models without manual strategies and achieves an 80\% linear scaling efficiency on Wide-ResNet~\cite{zagoruyko2016wide} with 4 nodes. This means developers can get efficient model-parallel execution of large DL models out-of-the-box using \sys. 

\begin{figure}[t]
	\centering
    \includegraphics[width=0.95\linewidth]{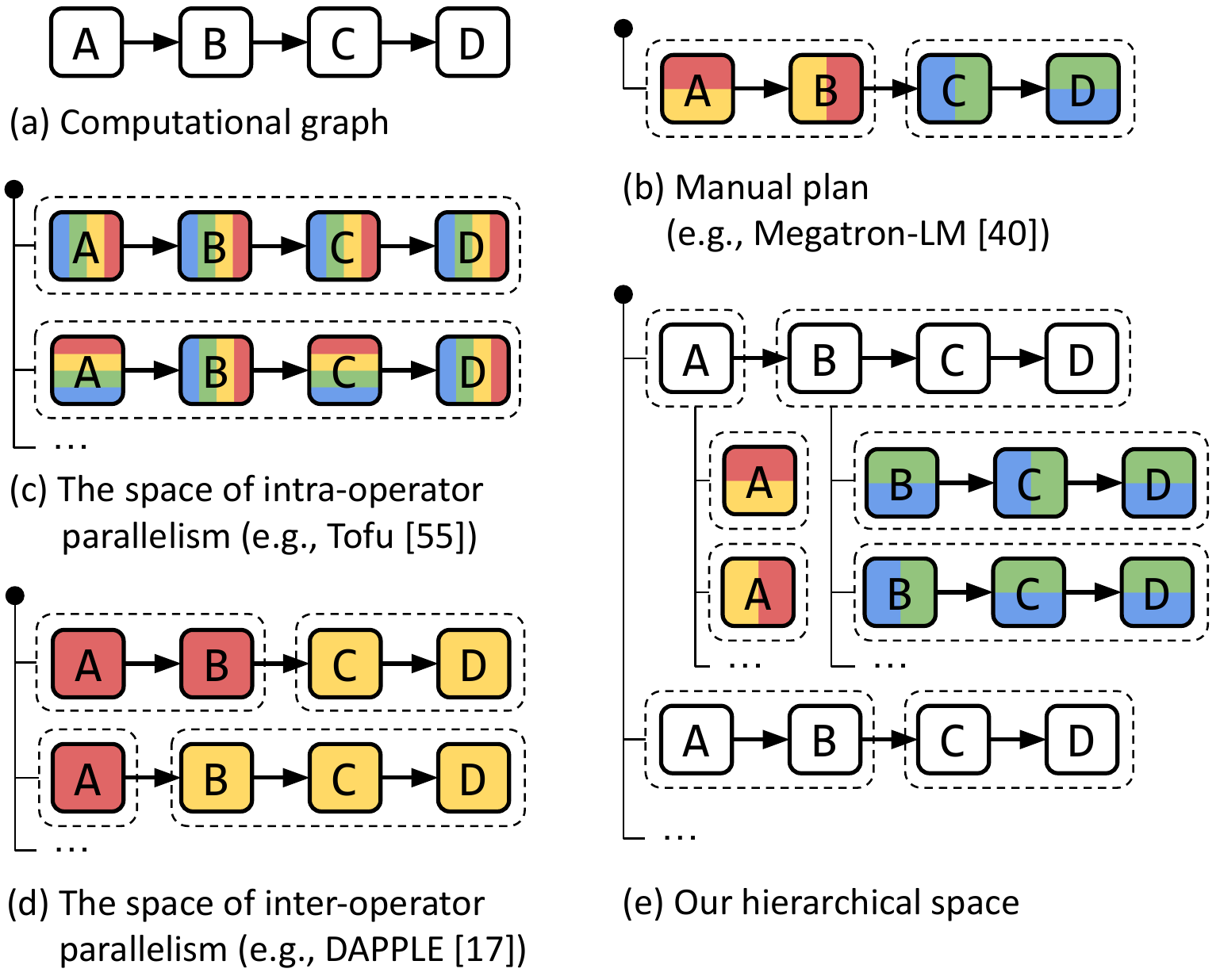}
    \vskip -0.0em
	\caption{Generation of parallelization plans for a computational graph shown in (a). Different colors represent different devices, dashed boxes represent pipeline stages. (b) creates the plan manually. (c) and (d) automatically generate plans using only one of intra- and inter-operator parallelisms. (e) shows our approach that creates a hierarchical space to combine intra- and inter-operator parallelisms.}
	\vskip -1.0em
    \label{fig:all-plans}
\end{figure}


\section{Background: Distributed Deep Learning}
\label{sec:background}

\begin{figure*}[t]
	\centering
	\includegraphics[width=0.95\textwidth]{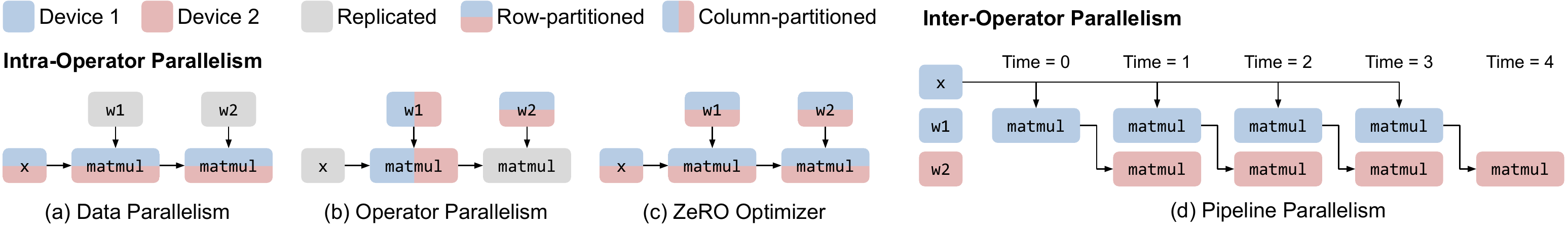}
	\vskip -0.5em
	\caption{Common parallelization techniques for training a 2-layer Multi-layer Perceptron (MLP). Only the forward pass is shown. ``x'' is the input data. ``w1'' and ``w2'' are two weight matrices.}
	\vskip -1em
	\label{fig:parallel-patterns}
\end{figure*}

DL computation is commonly represented by popular ML frameworks~\cite{abadi2016tensorflow,paszke2019pytorch,jax2018github} as a dataflow graph. Edges in the graph represent multi-dimensional tensors; nodes are computational operators, such as matrix multiplication (matmul), that transform input tensors into output tensors. 
Training a DL model for one iteration consists of computing a loss by \emph{forwarding} a batch of data through the graph, deriving the updates via a reverse \emph{backward} pass, and applying the updates to the parameters via \emph{weight update} operations.
In practice, model developers define the dataflow graph. An execution engine then optimizes and executes it on a compute device.

When either the model or data is large that a single device cannot complete the training in a reasonable amount of time, we resort to ML parallelization approaches that parallelize the computation on distributed devices.

\subsection{Conventional View of ML Parallelism}
\label{sec:conventional}
Existing ML parallelization approaches are typically categorized as data, operator, and pipeline parallelisms.

\noindent \textbf{Data parallelism.} In data parallelism, the training data is partitioned across distributed workers, but the model is replicated. Each worker computes the parameter updates on its independent data split, and synchronizes the updates with other workers before the weight update, so that all workers observe consistent model parameters throughout training.

\noindent \textbf{Operator parallelism.} When the model is too large to fit in one device, operator parallelism is an effective model parallelism option. Operator parallelism refers to approaches that partition the computation of a specific operator (abbreviated as \emph{op} in the following text), such as matmul shown in Fig.~\ref{fig:parallel-patterns}b, along \emph{non-batch} axes, and compute each part of the operator in parallel across multiple devices. 

Because input tensors are jointly partitioned, when a device computes its op partition, the required portions of input tensors may not reside in its local memory. Communication is thus required to fetch the input data from other devices.
When the tensors are partitioned evenly, i.e., SPMD\cite{xu2021gspmd}, all devices will follow the same collective communication patterns such as all-reduce, all-gather, and all-to-all.

\noindent \textbf{Pipeline parallelism.} Instead of partitioning ops, pipeline parallelism places different groups of ops from the model graph, referred as \emph{stages}, on different workers; meanwhile, it splits the training batch as a number of microbatches, and pipelines the forward and backward passes across microbatches on distributed workers, as Fig.~\aref{fig:parallel-patterns}{d} shows.
Unlike operator parallelism, pipeline parallelism transfers intermediate activations at the forward and backward passes between different workers using point-to-point communication.

\noindent \textbf{Manual combination of parallelisms.}
Recent development shows the approaches mentioned above need to be combined to scale out today's large DL models\cite{narayanan2021efficient, xu2021gspmd}.
The state-of-the-art training systems, such as Megatron-LM~\cite{shoeybi2019megatron,narayanan2021efficient}, manually design a specialized execution plan that combines these parallelisms for transformer language models, which is also known as \emph{3D Parallelism}. By assuming the model has the same transformer layer repeated, it assigns an equal number of layers to each pipeline stage and applies a hand-designed operator and data parallelism configuration uniformly for all layers. Despite the requirement of strong expertise, the manual plan cannot generalize to different models or different cluster setups (\S\ref{subsec:e2e}).

\noindent \textbf{Automatic combination of parallelisms.}
The configurations of each individual parallelism, their interdependence, and their dependence on model and cluster setups form an intractable space, which prevents the trivial realization of automatically combining these parallelisms.
For examples, when coupled with operator parallelism, each time adding a data-parallel replica would require  allocating a new set of devices (instead of one single device) as the worker, and figuring out the optimal operator parallelism configurations within those devices. When including pipeline parallelism, the optimal pipelining scheme depends on the data and operator parallelism choices of each pipeline stage and how devices are allocated for each stage. 
With this conventional view, prior explorations~\cite{jia2018beyond,wang2019supporting,fan2021dapple,zhang2020autosync} of auto-parallelization are limited to combining data parallelism with at most one model parallelism approach, which misses substantial performance opportunities. We next develop our view of ML parallelisms.

\subsection{Intra- and Inter-Operator Parallelisms}
\label{sec:background:intra-inter}
Different from the conventional view, in this paper, we re-catalog existing parallelization approaches into two orthogonal categories: intra-operator and inter-operator parallelisms. They are distinguished by if they involve partitioning operators along any tensor axis. 
We next use the examples in Fig.~\ref{fig:parallel-patterns} to introduce the two types of parallelisms.
 
\noindent \textbf{Intra-operator parallelism.}
An operator works on multi-dimensional tensors. We can partition the tensor along some dimensions, assign the resulting partitioned computations to multiple devices, and let them execute different portions of the operator at the same time. We define all parallelization approaches using this workflow as intra-operator parallelism.

Fig.~\aref{fig:parallel-patterns}{a-c} illustrates the application of several typical instantiations of intra-op parallelism on an MLP. 
Data parallelism~\cite{krizhevsky2014one}, by definition, belongs to intra-op parallelism -- the input tensors and matmuls are partitioned along the batch dimension, and weight tensors are replicated. 
Alternatively, when the weights are very large, partitioning the weights (Fig.~\aref{fig:parallel-patterns}{b}) leads to the aforementioned operator parallelism adopted in Megatron-LM.
Besides operators in the forward or backward passes, one can also partition the operators from the weight update phase, yielding the \emph{weight update sharding} or equivalently the ZeRO~\cite{rajbhandari2020zero, xu2020automatic} technique, commonly comprehended as an optimization of data parallelism.

Due to the partitioning, collective communication is required at the split and merge of the operator. Hence, a key characteristic of intra-operator parallelism is that it results in substantial communication among distributed devices.

\noindent \textbf{Inter-operator parallelism.}
We define inter-operator parallelism as the orthogonal class of approaches that \emph{do not} perform operator partitioning, but instead, assign different operators of the graph to execute on distributed devices.

Fig.~\aref{fig:parallel-patterns}{d} illustrates the batch-splitting pipeline parallelism as a case of inter-operator parallelism.\footnote{Device placement~\cite{mirhoseini2017device} is another case of inter-op parallelism, which partitions the model graph and executes them on different devices but does not saturate pipelines using multiple microbatches. Hence pipeline parallelism is often seen as a better alternative to it because of less device idle time.}
The pipeline execution can follow different schedules, such as Gpipe~\cite{huang2019gpipe}, PipeDream~\cite{narayanan2019pipedream}, and synchronous 1F1B~\cite{narayanan2021memory, fan2021dapple}. We adopt the synchronous 1F1B schedule throughout this paper as it respects synchronous consistency, and has the same pipeline latency but lower peak memory usage compared to Gpipe.

In inter-operator parallelism, devices communicate only between pipeline stages, typically using point-to-point communication between device pairs.
The required communication volume can be much less than the collective communication in intra-operator parallelism.
Regardless of the schedule used, due to the data dependency between stages, inter-operator parallelism results in some devices being idle during the forward and backward computation.

By this categorization, the two parallelisms take place at different granularities of the DL computation and have distinct communication requirements, which happen to match the structure of today's compute clusters. We will leverage these properties to design hierarchical algorithms and compilation passes to auto-generate execution plans. Several concurrent work \cite{li2021terapipe, narayanan2021memory, athlur2022varuna, tarnawski2021piper} have proposed similar categorization, but \sys is the first end-to-end system that uses this categorization to automatically generate parallel plans from the full space.

\section{Overview}
\label{sec:overview}

\begin{figure}[t]
	\centering
	\includegraphics[width=\columnwidth]{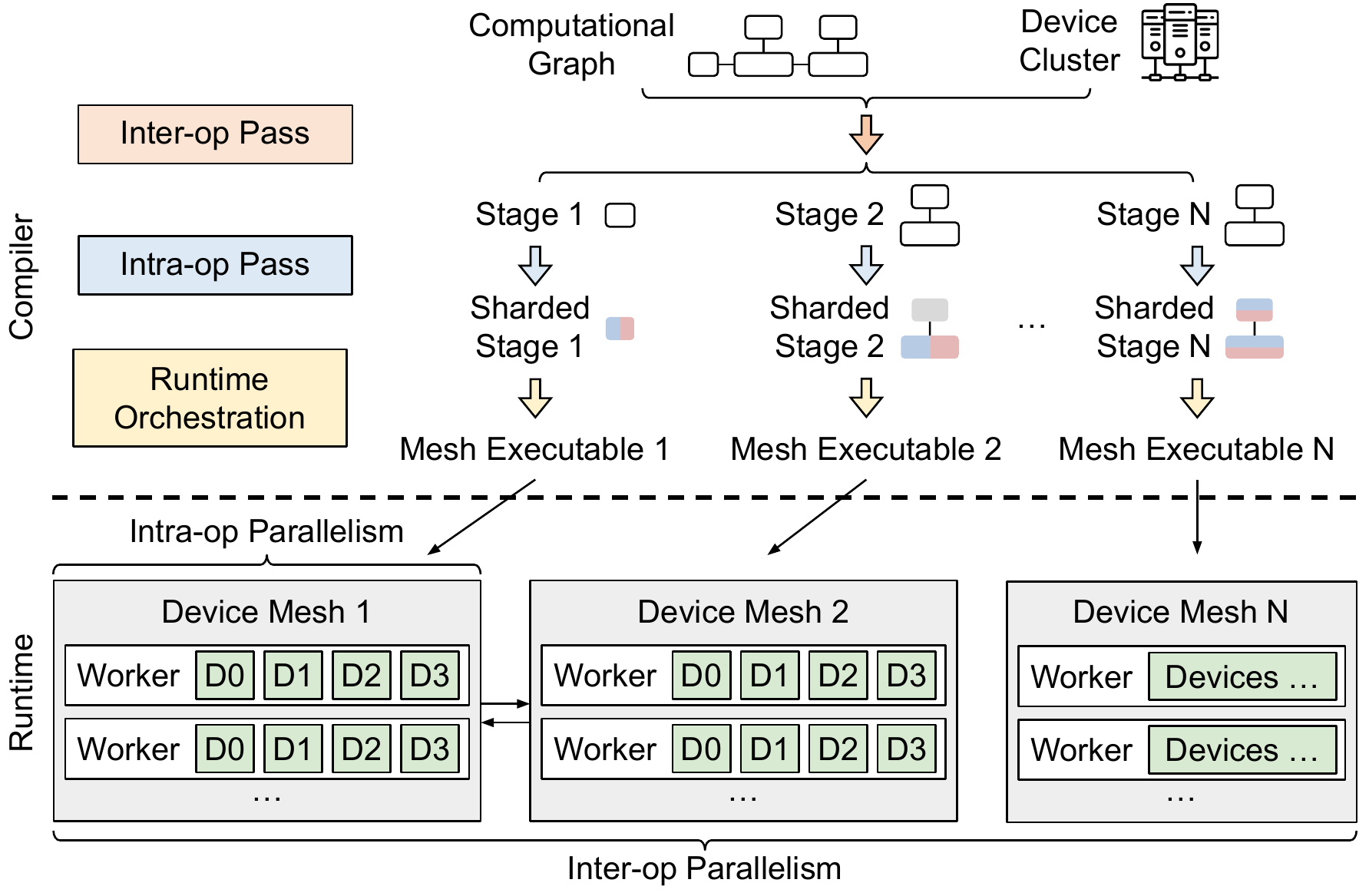}
    \vskip -0.8em
	\caption{Compiler passes and runtime architecture. A sharded stage is a stage annotated with the sharding specs generated by intra-op pass.}
	\label{fig:overview}
	\vskip -0.5em
\end{figure}

\begin{figure}[t]
	\centering
	\includegraphics[width=\columnwidth]{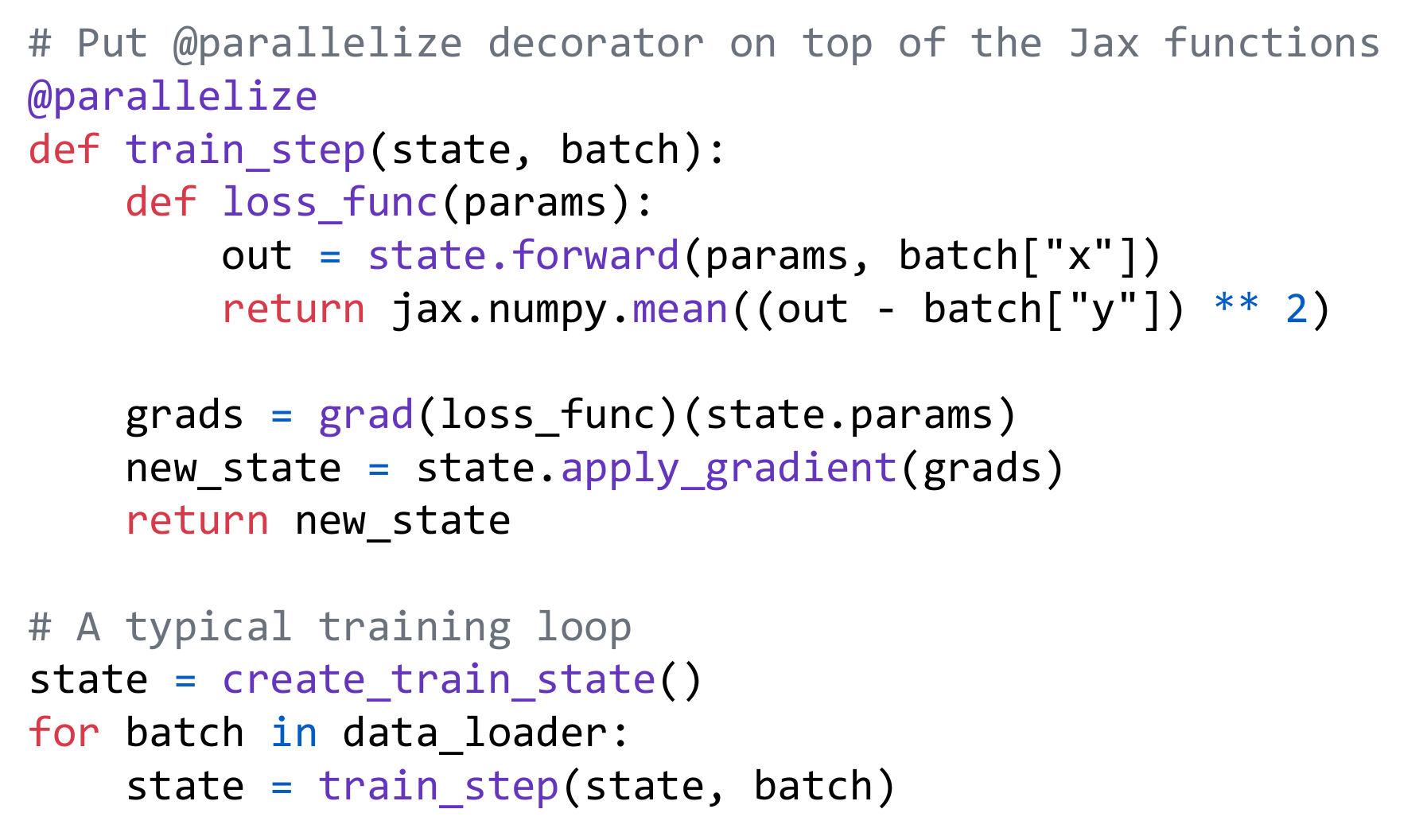}
	\vskip -0.8em
	\caption{An example to demonstrate \sys's API for Jax. The developers uses a Python decorator \texttt{@parallelize} to annotate functions that need to be parallelized. The rest of the program is kept intact.}
	\vskip -1.0em
	\label{fig:api_example}
\end{figure}

\sys is a compiler that generates model-parallel execution plans by hierarchically optimizing the plan at two different levels: intra-op and inter-op parallelism. 
At the intra-op level, \sys minimizes the cost of executing a stage (i.e., subgraph) of the computational graph, with respect to its intra-operator parallelism plan, on a given device mesh, which is a set of devices that may have high bandwidth between each other (e.g., GPUs within a single server). Different meshes might have different numbers of computing devices according to the workload assigned.
At the inter-op level, \sys minimizes the inter-op parallelization latency, with respect to how to slice the model and device cluster into stages and device meshes and how to map them as stage-mesh pairs. The inter-op optimization depends on knowing the execution cost of each stage-mesh pair reported by the intra-op optimizer. Through this hierarchical optimization process, \sys generates the execution plan consisting of intra-op and inter-op plans which are locally near-optimal at their respective level of the hierarchy.

To achieve this, \sys implements three novel compilation passes as Fig.~\ref{fig:overview} shows.
Given a model description, in the form of a Jax~\cite{jax2018github} intermediate representation (IR), and a cluster configuration, the inter-op compilation pass slices the IR into a number of stages, and slices the device cluster into a number of device meshes.
The inter-op pass uses a Dynamic Programming (DP) algorithm to assign stages to meshes and invokes the intra-op compilation pass on each stage-mesh pair, to query the execution cost of this assignment. Once invoked, the intra-op pass optimizes the intra-op parallel execution plan of the stage running on its assigned mesh, by minimizing its execution cost using an Integer Linear Programming (ILP) formulation, and reports the cost back to the inter-op pass. By repeatedly querying the intra-op pass for each allocation of a stage-mesh pair, the inter-op pass uses the DP to minimize the inter-op parallel execution latency and obtains the best slicing scheme of stages and meshes.

Given the output hierarchical plan and a designated pipeline-parallel schedule, each stage is first compiled as a parallel executable on its located mesh. A runtime orchestration pass is invoked to fulfill the communication requirement between two adjacent stages that require communication between the two meshes they locate on.
The runtime orchestration pass then generates static instructions specific to each mesh according to the pipeline-parallel schedule and invokes the execution on all meshes.

\noindent \textbf{API.} \sys has a simple API shown in Fig.~\ref{fig:api_example}. 
\sys requires developers to annotate functions to be parallelized, such as the {\tt train\_step()}, using a Python decorator {\tt @parallelize}.
Upon the first call to \texttt{train\_step()}, \sys traces the whole function to get the model IR, invokes the compilation, and converts the function to a parallel version.

Since the inter-op pass depends on the intra-op pass, in the following text, we first describe the intra-op pass, followed by the inter-op pass, and finally the runtime orchestration pass.


\section{Intra-Operator Parallelism}
\label{sec:intra-op}

\sys optimizes the intra-operator parallelism plan within a device mesh. \sys adopts the SPMD-style intra-op parallelism ~\cite{lepikhin2020gshard,xu2021gspmd} which partitions operators evenly across devices and executes the same instructions on all devices, as per the fact that devices within a single mesh have equivalent compute capability.
This SPMD style significantly reduces the space of intra-op parallelism plans; meanwhile, it conveniently expresses and unifies many important approaches such as data parallelism, ZeRO, Megatron-LM's operator parallelism, and their combinations, which are not fully covered by existing automatic operators parallelism systems, such as Tofu~\cite{wang2019supporting} and FlexFlow~\cite{jia2018beyond}.
Unlike systems that perform randomized search~\cite{jia2018beyond} or assume linear graphs~\cite{wang2019supporting}, \sys formalizes the problem as an integer linear programming (ILP) and shows it can be solved efficiently for computational graphs with tens of thousands of operators.
Next, we describe the space of intra-op parallelism and our solution.

\begin{table}[t]
	\renewcommand{\arraystretch}{1.4}
	\setlength{\tabcolsep}{3pt}
	\centering
	\footnotesize
	\caption{Sharding specs of a 2-dimentional tensor on a $2\times2$ device mesh. $A$ is a $(N, M)$ tensor. The device mesh is [[Device 0, Device 1], [Device 2, Device 3]]. Each device stores a partition of $A$. The first column is the name of the sharding spec. The latter columns use Numpy syntax to describe the partitions stored on each device.}
	\vskip -0.5em
	\begin{tabular}{ccccc}
		\toprule
		Spec & Device 0 & Device 1 & Device 2 & Device 3 \\
		\midrule
		$RR$ & $A[0:N, 0:M]$ & $A[0:N, 0:M]$ & $A[0:N, 0:M]$ & $A[0:N, 0:M]$ \\
		$S^0S^1$ & $A[0:\frac{N}{2}, 0:\frac{M}{2}]$ & $A[0:\frac{N}{2}, \frac{M}{2}:M]$ & $A[\frac{N}{2}:N, 0:\frac{M}{2}]$ & $A[\frac{N}{2}:N, \frac{M}{2}:M]$ \\
		$S^1S^0$ & $A[0:\frac{N}{2}, 0:\frac{M}{2}]$ & $A[\frac{N}{2}:N, 0:\frac{M}{2}]$ & $A[0:\frac{N}{2}, \frac{M}{2}:M]$ & $A[\frac{N}{2}:N, \frac{M}{2}:M]$ \\
		$S^0R$ & $A[0:\frac{N}{2}, 0:M]$ & $A[0:\frac{N}{2}, 0:M]$ & $A[\frac{N}{2}:N, 0:M]$ & $A[\frac{N}{2}:N, 0:M]$ \\
        $S^1R$ & $A[0:\frac{N}{2}, 0:M]$ &$A[\frac{N}{2}:N, 0:M]$ & $A[0:\frac{N}{2}, 0:M]$ & $A[\frac{N}{2}:N, 0:M]$ \\
        $RS^0$ & $A[0:N, 0:\frac{M}{2}]$ & $A[0:N, 0:\frac{M}{2}]$ & $A[0:N, \frac{M}{2}:M]$ & $A[0:N, \frac{M}{2}:M]$ \\
	    $RS^1$ & $A[0:N, 0:\frac{M}{2}]$ &$A[0:N, \frac{M}{2}:M]$ & $A[0:N, 0:\frac{M}{2}]$ & $A[0:N, \frac{M}{2}:M]$ \\
        $S^{01}R$ & $A[0:\frac{N}{4}, 0:M]$ &$A[\frac{N}{4}:\frac{N}{2}, 0:M]$ & $A[\frac{N}{2}:\frac{3N}{4}, 0:M]$ & $A[\frac{3N}{4}:N, 0:M]$ \\
        $RS^{01}$ & $A[0:N,0:\frac{M}{4}]$ &$A[0:N,\frac{M}{4}:\frac{M}{2}]$ & $A[0:N, \frac{M}{2}:\frac{3M}{4}]$ & $A[0:N, \frac{3M}{4}:M]$ \\
		\bottomrule
	\end{tabular}
	\vskip -1em
	\label{table:sharding-specs}
\end{table}

\begin{table}[t]
	\renewcommand{\arraystretch}{1.2}
	\centering
	\footnotesize
	\caption{Several cases of resharding. $all\mhyphen gather(x, i)$ means an all-gather of $x$ bytes along the $i$-th mesh axis. $M$ is the size of the tensor. $(n_0, n_1)$ is the mesh shape.}
	\vskip -0.5em
	\begin{tabular}{cccc}
		\toprule
		\# & Src Spec & Dst Spec & Communication Cost \\
		\midrule
		1 & $RR$ & $S^0S^1$ & $0$ \\
		2 & $S^0R$ & $RR$ & $all\mhyphen gather (M, 0)$ \\
		3 & $S^0S^1$ & $S^0R$ & $all\mhyphen gather (\frac{M}{n_0}, 1)$ \\
		4 & $S^0R$ & $RS^0$ & $all\mhyphen to \mhyphen all(\frac{M}{n_0}, 0)$ \\
		5 & $ S^0S^1$ &  $S^{01}R$ & $all\mhyphen to \mhyphen all(\frac{M}{n_0 \cdot n_1}, 1)$ \\
		\bottomrule
	\end{tabular}
	\vskip -1.5em
	\label{table:resharding-cost}
\end{table}

\subsection{The Space of Intra-Operator Parallelism}
Given an operator in the computational graph, there are multiple possible parallel algorithms to run it on a device mesh.
For example, a matrix multiplication $C_{ij}=\sum_k A_{ik}B_{kj}$ corresponds to a three-level for-loop. To parallelize it, we can parallelize the loop $i$, loop $j$, loop $k$, or combinations of them across devices, which would have different computation and communication costs,  require different layouts for the input tensors, and result in output tensors with different layouts.
If an input tensor does not satisfy the layout requirement, a layout conversion is required, which introduces extra communication costs.
The goal of the intra-op pass is to pick one parallel algorithm for every operator to minimize the execution time of the entire graph.
Next, we formally define the device mesh and the layout of a tensor and discuss the cost of layout conversion.

\noindent \textbf{Device mesh}. A device mesh is a 2-dimensional logical view of a set of physical devices. Each device in the mesh has the same compute capability. Devices can communicate along the first mesh dimension and the second mesh dimension with different bandwidths. We assume different groups of devices along the same mesh dimension have the same communication performance. For a set of physical devices, there can be multiple logical views. For example, given 2 nodes and 8 GPUs per node (i.e., 16 devices in total), we can view them as a $2\times8$, $1\times 16, 4\times4, 8\times2$, or $16\times 1$ device mesh. The mapping between physical devices and the logical device mesh view is optimized by the inter-op pass (\S\ref{sec:inter-op}). In the rest of this section, we consider one fixed device mesh view.

\noindent \textbf{Sharding Spec}. We use \emph{sharding spec} to define the layout of a tensor.
For an $N$-dimensional tensor, its sharding spec is defined as $X_0X_1 \cdots X_{n-1}$, where $X_i \in \{S, R\}$. If $X_i = S$, it means the $i$-th axis of the tensor is partitioned. Otherwise, the $i$-th axis is replicated. 
For example, for a 2-dimensional tensor (i.e., a  matrix), $SR$ means it is row-partitioned, $RS$ means it is column-partitioned, $SS$ means it is both row- and column- partitioned. $RR$ means it is replicated without any partitioning.
After we define which tensor axes are partitioned, we then have to map the partitioned tensor axes to mesh axes. We only consider 2-dimensional device meshes, so a partitioned tensor axis can be mapped to either the first or the second axis of the device mesh, or both.
We added a superscript to $S$ to denote the device assignment. For example, $S^0$ means the partitions are along the $0$-th axis of the mesh,
$S^{01}$ means the partitions take place along both mesh axes.
$S^0R$ means the tensor is row-partitioned into two parts -- The first part is replicated on device 0 and device 1, and the second part is replicated on device 2 and device 3.
Table~\ref{table:sharding-specs} shows all possible sharding specs of a 2-dimensional tensor on a $2 \times 2$ mesh with 4 devices.

\noindent \textbf{Resharding.} When an input tensor of an operator does not satisfy the sharding spec of the chosen parallel algorithm for the operator, a layout conversion, namely \emph{resharding}, is required, which might require cross-device communication. 
Table~\ref{table:resharding-cost} lists several cases of resharding. For instance, to convert a fully replicated tensor to any other sharding specs (case \#1), we can slice the tensor locally without communication; to swap the partitioned axis (case \#4), we perform an all-to-all.

\begin{table}[t]
	\renewcommand{\arraystretch}{1.0}
	\centering
	\setlength{\tabcolsep}{5pt}
	\footnotesize
	\caption{Several parallel algorithms for a batched matmul $C_{b,i,j} = \sum_k A_{b,i,k} B_{b,k,j}$. The notation $all\mhyphen reduce(x, i)$ means an all-reduce of $x$ bytes along the $i$-th mesh axis. $M$ is the size of the output tensor. $(n_0, n_1)$ is the mesh shape.}
	\vskip -1em
	\begin{tabular}{ccccc}
		\toprule
		\multirow{2}{*}{\#} & Parallel   & Output  & Input   & Communication   \\
		& Mapping & Spec     & Specs & Cost \\
		\midrule
		1 & $i \rightarrow 0, j \rightarrow 1$
		& $RS^0S^1$
		&	$RS^0R,RRS^1$ & 0 \\
		2 & $i \rightarrow 0, k \rightarrow 1$  & $RS^0R$      &   $RS^0S^1,RS^1R$ & $all\mhyphen reduce(\frac{M}{n_0}, 1)$ \\
		3 & $j \rightarrow 0, k \rightarrow 1$  & $RRS^0$      &   $RRS^1,RS^1S^0$ & $all\mhyphen reduce(\frac{M}{n_0}, 1)$ \\
		4 & $b \rightarrow 0, i \rightarrow 1$  & $S^0S^1R$      &   $S^0S^1R,S^0RR$ & 0 \\
		5 & $b \rightarrow 0, k \rightarrow 1$  & $S^0RR$      &   $S^0RS^1,S^0S^1R$ & $all\mhyphen reduce(\frac{M}{n_0}, 1)$ \\
		6 & $i \rightarrow \{0, 1\}$  & $RS^{01}R$      &   $RS^{01}R,RRR$ & 0 \\
		7 & $k \rightarrow \{0, 1\}$  & $RRR$      &   $RRS^{01},RS^{01}R$ & $all\mhyphen reduce(M, \{0,1\})$ \\
		\bottomrule
	\end{tabular}
	\vskip -1.5em
	\label{table:dot-algorithms}
\end{table}

\noindent \textbf{Parallel algorithms of an operator.}
With the definitions above, consider parallelizing a batched matmul $C_{b,i,j} = \sum_k A_{b,i,k} B_{b,k,j}$ on a 2D mesh -- Table~\ref{table:dot-algorithms} lists several intra-op parallel algorithms for a batched matmul.
Algorithm\#1 maps loop $i$ to the $0$-th mesh axis and loop $j$ to the $1$-th mesh axis,
resulting in the output tensor $C$ with a sharding spec $RS^0S^1$.
As the LHS operand $A_{b,i,k}$ and RHS operand $B_{b,k,j}$ both have only one parallelized index, their sharding specs are written as $RS^0R$ and $RRS^1$, respectively.
In this algorithm, each device has all its required input tiles (i.e., a partition of the tensor) stored locally to compute its output tile, so there is no communication cost.
In Algorithm \#2 in Table~\ref{table:dot-algorithms}, when the reduction loop $k$ is parallelized, all-reduce communication is needed to aggregate the partial sum.
Similarly, we can derive the sharding specs and communication costs of other parallel algorithms for a batched matmul.

For other primitive operators such as convolution and reduction, we can get a list of possible parallel algorithms following a similar analysis of their math expressions.
In the intra-op pass, the model graph is represented in XLA's HLO format\cite{xla2017google}, which summarizes common DL operators into less than 80 primitive operators, so we can manually enumerate the possible parallel algorithms for every primitive operator.

\subsection{ILP Formulation}
The total execution cost of a computational graph $G = (V, E)$ is the sum of the compute and communication costs on all nodes $v \in V$ and the resharding costs on all edges $e \in E$.
We formulate the cost minimization as an ILP and solve it optimally with an off-the-shelf solver \cite{forrest2005cbc}.

For node $v$, the number of possible parallel algorithms is $k_v$. It then has a communication cost vector $c_v$ of length $k_v$, or $c_v \in \mathbb{R}^{k_v}$, where $c_{vi}$ is the communication cost of the $i$-th algorithm. Similarly, node $v$ has a compute cost vector $d_v \in \mathbb{R}^{k_v}$.
For each node $v$, we define an one-hot decision vector $s_v \in \{0,1\}^{k_v}$ to represent the algorithm it uses.
$s_{vi} = 1$ means we pick the $i$-th algorithm for node $v$.
For the resharding cost between node $v$ and node $u$, we define a resharding cost matrix $R_{vu} \in \mathbb{R}^{k_v \times k_u}$, where $R_{vuij}$ is the resharding cost from the output of $i$-th strategy of node $v$ to the input of $j$-th strategy of node $u$. The objective of the problem is
\begin{equation}
\label{eq:ilp-objective}
\underset{s}{\text{min}} ~ \sum_{v \in V} s_v^\intercal(c_v + d_v) + \sum_{(v,u) \in E} s_v^\intercal R_{vu}s_u,
\end{equation}
where the first term is the compute and communication cost of node $v$, and the second is the resharding cost of the edge $(v,u)$. In this formulation, $s$ is the variable, and the rest are constant values.
The term $s_v^\intercal R_{vu}s_u $ in Eq.~\ref{eq:ilp-objective} is quadratic, and cannot be fed into an ILP solver. We linearize~\cite{forrester2020computational} the quadratic term by introducing a new decision vector $e_{vu} \in \{0,1\}^{k_v \cdot k_u}$ which represents the resharding decision between node $v$ and $u$.

Although we can use profiling to get the accurate costs for $c_v$, $d_v$, and $R_{vu}$, we use the following methods to estimate them for simplicity.
For communication costs $c_v$ and $R_{vu}$, we compute the numbers of communicated bytes and divide them by the mesh dimension bandwidth to get the costs.
For compute costs $d_v$, we set all of them as zero following the same motivation in \cite{wang2019supporting}.
This is reasonable because:
(1) For heavy operators such as matmul, we do not allow replicated computation. All parallel algorithms always evenly divide the work to all devices, so all parallel algorithms of one operator have the same arithmetic complexity;
(2) For lightweight operators such as element-wise operators, we allow replicated computation of them, but their computation costs are negligible.

To simplify the graph, we merge computationally-trivial operators, such as element-wise operators, transpose, and reduction, into one of their operands and propagate the sharding spec from the operand.
This greatly reduces the number of nodes in the graph, thus the ILP problem size. 
We do a breath-first-search and compute the depth of each node. The node is merged to the deepest operand.

Once the parallel plan is decided by ILP, we also apply a set of post-ILP communication optimizations, such as replacing all-reduce with reduce-scatter and all-gather, whenever applicable, because the latter reduces the number of replicated tensors and corresponding computations, while keeping the communication volume the same. This achieves the effect of weight update sharding~\cite{xu2020automatic} or ZeRO optimizer~\cite{rajbhandari2020zero}.


\section{Inter-Operator Parallelism}
\label{sec:inter-op}
In this section, we develop methods to slice the model and device cluster into stage-mesh pairs. Our optimization goal is to minimize the \emph{end-to-end} pipeline execution latency for the entire computational graph.
Previous works~\cite{li2021terapipe,fan2021dapple} have considered simplified problems, such as assuming the device for each stage is pre-assigned, and all stages have fixed data or operator parallelism plan. Alpa rids these assumptions by jointly considering device mesh assignment and the existence of varying intra-op parallelism plans on each stage.

\subsection{The Space for Inter-Operator Parallelism}

Assume the computational graph contains a sequence of operators following the topology order of the graph\footnote{We simply use the order of how users define each operator, reflected in the model IR, with the input operator as the origin. This allows us to leverage the inherent locality present in the user's program -- closely related nodes in the graph will be more likely to be partitioned into the same stage.}, notated as $o_1, \ldots, o_K,$ where the inputs of an operator $o_k$ are from operators $o_1, \ldots, o_{k-1}.$ We slice the operators into $S$ stages $s_1, \ldots, s_S,$ where each stage $s_i$ consists of operators $(o_{l_i}, \ldots, o_{r_i}),$ and we assign each stage $s_i$ to a submesh of size $n_i \times m_i$, sliced from a computer cluster that contains devices, notated as the \emph{cluster mesh} with shape $N \times M$.
Let $t_i = t_\mathit{intra}(s_i, \mathit{Mesh}(n_i, m_i))$ be the latency of executing stage $s_i$ on a submesh of $n_i \times m_i$, minimized by the ILP and reported back by the intra-op pass (\S\ref{sec:intra-op}).
As visualized in Fig.~\ref{fig:pipeline-latency-illustration}, assuming we have $B$ different input microbatches for the pipeline, the total minimum latency\footnote{This formulation holds for GPipe and synchronous 1F1B schedules. Other pipeline schedules may require a different formulation.} for the entire computation graph is written as:
\vskip-1.5em
\begin{equation}
\label{eq:pipeline-latency}
T^* = \min_{\substack{s_1, \ldots, s_S; \\ (n_1, m_1), \ldots, (n_S, m_S)}}{\left\{\sum_{i=1}^S t_i + (B - 1) \cdot \max_{1 \le j \le S}\{t_j\}\right\}}. 
\end{equation}
\vskip-0.5em

The overall latency contains two terms: the first term is the total latency of all stages, interpreted as the latency of the first microbatch going through the pipeline; the second term is the pipelined execution time for the rest of $B - 1$ microbatches, which is bounded by the slowest stage (stage 3 in Fig.~\ref{fig:pipeline-latency-illustration}).

We aim to solve Eq.~\ref{eq:pipeline-latency} with two additional constraints: (1) For an operator in the forward pass of the graph, we want to colocate it with its corresponded backward operator on the same submesh.
Since backward propagation usually uses the similar set of tensors during forward propagation, this effectively reduces the amount of communication to fetch the required tensors generated at the forward pass to the backward pass. We use the sum of forward and backward latency for $t_\mathit{intra}$, so Eq.~\ref{eq:pipeline-latency} reflects the total latency, including both forward and backward propagation. (2) We need the sliced submeshes $(n_1, m_1), \ldots, (n_S, m_S)$ to fully cover the $N \times M$ cluster mesh -- we do not waste any compute device resources. We next elaborate on our DP formulation.

\begin{figure}
    \centering
    \includegraphics[width=.95\columnwidth]{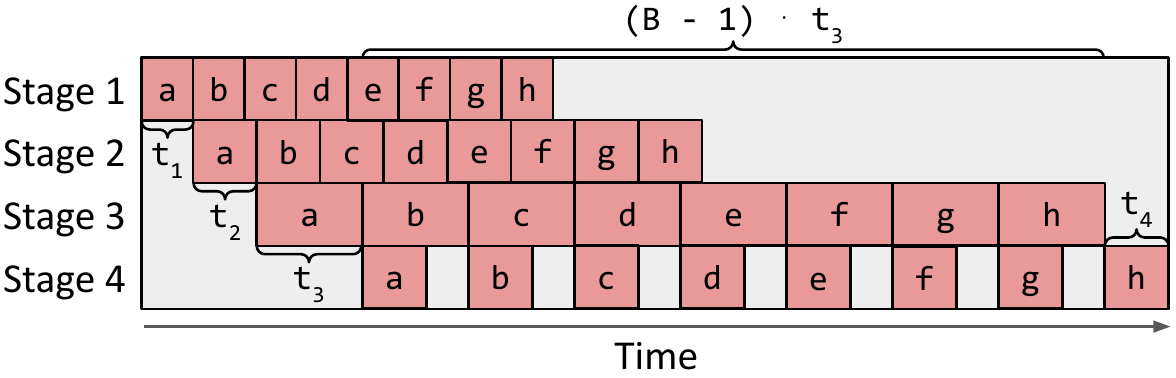}
    \vskip -0.5em
    \caption{Illustration of the total latency of a pipeline, which is determined by two parts: the total latency of all stages ($t_1 + t_2 + t_3 + t_4$) and the latency of the slowest stage ($(B-1)\cdot t_3$).}
    \vskip -0.5em
    \label{fig:pipeline-latency-illustration}
\end{figure}

\subsection{DP Formulation}
To ensure all submeshes $(n_1, m_1), \ldots, (n_S, m_S)$ fully cover the $N \times M$ cluster mesh, we reduce the available submesh shapes into two options: (1) one-dimensional submeshes of sizes $(1, 1), (1, 2), (1, 4) \ldots (1, 2^m)$ and (2) two-dimensional submeshes of size $(2, M), (3, M), \ldots, (N, M)$ that fully use the second dimension of the cluster mesh (i.e., on a GPU cluster, this means using all compute devices in each physical machine).
We include a theorem in Appendix~\ref{sec:proof_submesh_shape} that proves these submesh shapes can always fully cover the cluster mesh. To assign physical devices in the cluster to the resulting submeshes find by the DP algorithm, we enumerate by assigning devices to larger submeshes first and then to smaller ones. When there are multiple pipeline stages with the same submesh shape, we tend to put neighboring pipeline stages closer on the device mesh to reduce communication latency.

The simplification on submesh shapes works well for most available cloud deep learning setups: On AWS~\cite{aws}, the GPU instances have 1, 2, 4, or 8 GPUs; on GCP~\cite{gcp}, the TPU instances have 8, 32, 128, 256 or 512 TPUs. The set of submesh shapes $(n, m)$ excluded by the assumption is with $n > 1$ and $m < M$, which we observe lead to inferior results, since an alternative submesh with shape $(n', M)$ where $n'\cdot M = n\cdot m$ has more devices that can communicate with high bandwidth.
With this reduction, we only need to ensure that $\sum_{i=1}^S n_i \cdot m_i= N \cdot M$.

To find $T^*$ in Eq.~\ref{eq:pipeline-latency}, we develop a DP algorithm. The DP first enumerates the second term $t_\mathit{max} = \max_{1 \le j \le S}{t_j}$ and minimizes the first term $t_\mathit{total}(t_\mathit{max}) = \sum_{1 \le i \le S}{t_i}$ for each different $t_\mathit{max}$. Specifically, we use the function $F(s, k, d; t_\mathit{max})$ to represent the minimal total latency when slicing operators $o_k$ to $o_K$ into $s$ stages and putting them onto $d$ devices so that the latency of each stage is less than $t_\mathit{max}$. We start with $F(0, K + 1, 0; t_\mathit{max}) = 0$, and derive the optimal substructure of $F$ as
\begin{align}
\label{eq:dp}
    & F(s, k, d; t_\mathit{max}) \\ \nonumber 
    = & \min_{\substack{k \le i \le K \\ n_s \cdot m_s \le d}} \left\{
    \begin{aligned} 
    & t_\mathit{intra} ((o_k, \ldots, o_i), \mathit{Mesh}(n_s, m_s), s) \\ 
    & + F(s - 1, i + 1, d - n_s \cdot m_s; t_\mathit{max}) \\ 
    \mid {}&{} t_\mathit{intra} ((o_k, \ldots, o_i), \mathit{Mesh}(n_s, m_s), s) \le t_\mathit{max} 
    \end{aligned}
    \right\},
\end{align}
and derive the optimal total latency as
\begin{equation}
    T^*(t_\mathit{max}) = \min_{s}\{F(s, 0, N \cdot M; t_\mathit{max})\}  + (B - 1) \cdot t_\mathit{max}.
\end{equation}

The value of $t_\mathit{intra} ((o_k, \ldots, o_i), \mathit{Mesh}(n_s, m_s), s)$ is determined by the intra-op pass.
It is the lowest latency of executing the subgraph $(o_k, \ldots, o_i)$ on mesh $\mathit{Mesh}(n_s, m_s)$ with $s$ subsequent stages. Note that $\mathit{Mesh}(n_s, m_s)$ is a set of physical devices -- hence, we enumerate all the potential choices of logical device mesh shapes $(n_l, m_l)$ satisfying $n_l \cdot m_l = n_s \cdot m_s$.
For each choice, we query the intra-op pass with subgraph $(o_k, \ldots, o_i)$, logical mesh $(n_l, m_l)$, and other intra-op options as inputs and get an intra-op plan.
We then compile the subgraph with this plan and all other low-level compiler optimizations (e.g., fusion, memory planning) to get an executable for precise profiling.
The executable is profiled in order to get the stage latency $(t_l)$ and the memory required on each device to run the stage ($\mathit{mem}_\mathit{stage}$) and to store the intermediate activations ($\mathit{mem}_\mathit{act}$).
We check whether the required memory fits the device memory ($\mathit{mem}_\mathit{device}$) according to the chosen pipeline execution schedule. For example, for 1F1B schedule \cite{fan2021dapple, narayanan2021memory}, we check
\begin{equation}
    \mathit{mem}_\mathit{stage} + s\cdot \mathit{mem}_\mathit{act} \le \mathit{mem}_\mathit{device}.
    \label{eq:memory}
\end{equation}
We pick the logical mesh shape that minimizes $t_l$ and fits into the device memory. If none of them  fits, we set $t_\mathit{intra} = \infty$. 

Our algorithm builds on top of that in TeraPipe~\cite{li2021terapipe}. However, TeraPipe assumes all pipeline stages are the same, and the goal is to find the optimal way to batch input tokens into micro-batches of different sizes. Instead, \sys aims to group the operators of a computational graph into different pipeline stages, while assuming the input micro-batches are of the same size. In addition, \sys optimizes the mesh shape in the DP algorithm for each pipeline stage in inter-op parallelism.

\begin{algorithm}[t!]
  \caption{Inter-op pass summary.}
  \label{alg:inter-op}
\begin{algorithmic}[1]
    \STATE {\bfseries Input:} Model graph $G$ and cluster $C$ with shape $(N, M)$.
    \STATE {\bfseries Output:} The minimal pipeline execution latency $T^*$.
    \STATE \emph{// Preprocess graph.}
    \STATE $(o_1, \ldots, o_K) \leftarrow \text{Flatten}(G)$
    \STATE $(l_1, \ldots, l_L) \leftarrow \text{OperatorClustering}(o_1, \ldots, o_K)$
    \STATE  \emph{// Run the intra-op pass to get costs of different stage-mesh pairs.}
    \STATE $\mathit{submesh\_shapes} \leftarrow \{(1, 1), (1, 2), (1, 4), \ldots, (1, M)\} \cup \{(2, M), (3, M), \ldots, (N, M)\}$
    \FOR{$1 \le i \le j \le L$}
        \STATE $\mathit{stage} \leftarrow (l_i, \ldots, l_j)$
        \FOR{$(n, m) \in \mathit{submesh\_shapes}$}
            \FOR{$s$ \textbf{from} $1$ \textbf{to} $L$}
                \STATE $t\_\mathit{intra}(\mathit{stage}, \mathit{Mesh}(n, m), s) \leftarrow \infty$
            \ENDFOR
            \FOR{$(n_l, m_l), \mathit{opt} \in  \text{LogicalMeshShapeAndIntraOp}$  $\text{Options}(n, m)$}
                \STATE $\mathit{plan} \leftarrow \text{IntraOpPass}(\mathit{stage},  \mathit{Mesh}(n_l, m_l), \mathit{opt})$
                \STATE $t_l, \mathit{mem}_\mathit{stage}, \mathit{mem}_\mathit{act} \leftarrow \text{Profile}(plan)$
                \FOR{$s$ satisfies Eq.~\ref{eq:memory}}
                    \IF{$t_l < t\_\mathit{intra}(\mathit{stage}, \mathit{Mesh}(n, m), s)$}
                        \STATE                         $t\_\mathit{intra}(\mathit{stage}, \mathit{Mesh}(n, m), s) \leftarrow t_l$
                    \ENDIF
                \ENDFOR
            \ENDFOR
        \ENDFOR
    \ENDFOR
    \STATE \emph{// Run the inter-op dynamic programming}
    \STATE $T^* \leftarrow \infty$
    \FOR{$t_\mathit{max} \in \text{SortedAndFilter}(t\_\mathit{intra}, \varepsilon)$}
        \IF{$B \cdot t_\mathit{max} \ge T^*$}
            \STATE \textbf{break}
        \ENDIF
        \STATE $F(0, L+1, 0; t_\mathit{max}) \leftarrow 0$
        \FOR{$s$ \textbf{from} $1$ \textbf{to} $L$}
            \FOR{$l$ \textbf{from} $L$ \textbf{down to} $1$}
                \FOR{$d$ \textbf{from} $1$ \textbf{to} $N\cdot M$}
                \STATE Compute $F(s, l, d; t_\mathit{max})$ according to Eq.~\ref{eq:dp}
                \ENDFOR
            \ENDFOR
        \ENDFOR
        \STATE $T^*(t_\mathit{max}) \leftarrow \min_{s}\{F(s, 0, N \cdot M; t_\mathit{max})\} + (B - 1) \cdot t_\mathit{max}$
        \IF{$T^*(t_\mathit{max}) < T^*$}
            \STATE $T^* \leftarrow T^*(t_\mathit{max})$
        \ENDIF
    \ENDFOR
\end{algorithmic}
\end{algorithm}

\noindent \textbf{Complexity.} Our DP algorithm computes the slicing in $O(K^3NM(N+\log(M)))$ time for a fixed $t_\mathit{max}$. $t_\mathit{max}$ has at most $O(K^2(N+\log(M)))$ choices: $t_\mathit{intra}((o_i, \ldots, o_j), \mathit{Mesh}(n_s, m_s))$ for $i, j = 1, \ldots, K$ and all the submesh choices. The complexity of this DP algorithm is thus $O(K^5NM(N+\log(M))^2)$. 

This complexity is not feasible for a large computational graph of more than ten thousand operators. To speed up this DP, we introduce a few practical optimizations.

\noindent \textbf{Performance optimization \#1: early pruning.}
We use one optimization that is similar to that in TeraPipe~\cite{li2021terapipe}.
We enumerate $t_\mathit{max}$ from small to large. When $B \cdot t_\mathit{max}$ is larger than the current best $T^*$, we immediately stop the enumeration. This is because larger $t_\mathit{max}$ can no longer provide a better solution. Also, during enumeration of $t_\mathit{max},$ we only evaluate a choice of $t_\mathit{max}$ if it is sufficiently larger than the last $t_\mathit{max}$ (by at least $\epsilon$).
This allows the gap between the solution found by the DP algorithm and the global optima to be at most $B\cdot \epsilon$.
We empirically choose $\epsilon=10^{-6}\text{\,s}$, and we find that the solution output by our algorithm is the same as the real optimal solution ($\epsilon=0$) for all our evaluated settings.  

\noindent \textbf{Performance optimization \#2: operator clustering.}
Many operators in a computational graph are not computationally intensive (e.g., ReLU), and the exact placement of these operators has little impact on the total execution time.
We develop another DP algorithm~\cite{aydin2016distributed} to cluster neighboring operators to reduce the total size of the graph Eq.~\ref{eq:pipeline-latency} works on. We cluster the operators $(o_1, \ldots, o_K)$ into a series of layers\footnote{Note that the clustering does not exactly reproduce the layers with original machine learning semantics in the model definition.} $(l_1, \ldots, l_L),$ where $L \ll K.$ The goal of the algorithm is to merge two types of operators: (1) those that do not call for much computation but lengthen the computational graph and (2) neighboring operators that may cause substantial communication if put on different device meshes.
We define function $G(k, r)$ as the minimum of maximal amount of data received by a single layer when clustering operators $(o_1, \ldots, o_k)$ into $r$ layers. Note that $G$ has the following optimal substructure:
\begin{align}
    & G(k, r) \\
    = & \min_{1 \le i \le k} \left\{\begin{aligned}
        &\max\{G(i - 1, r - 1), C(i, k)\} \\ \biggm|{} & \mathit{FLOP}(o_i, \ldots, o_k) \le \frac{ (1 + \delta) \mathit{FLOP}_{\mathit{total}}}{L}
    \end{aligned} \right\}, \nonumber
\end{align}
where $C(i, k)$ denotes the total size of inputs of $(o_i,$ $\ldots, o_k)$ received from $(o_1, \ldots, o_{i-1})$ and $\mathit{FLOP}_\mathit{total} $ $= \mathit{FLOP}(o_1, \ldots, o_K)$ is the total FLOP of the whole computational graph. We make sure that each clustered layer's FLOP is within $1 + \delta$ times of the average FLOP per layer while minimizing the communication. For the solutions with the same communication cost, we choose the one with the most uniform structure by also minimizing the variance of per-layer FLOP.
With our DP algorithm, we can compute the best layer clustering in $O(K^2L)$ time. Note that $L$ here is a hyperparameter to the algorithm. In practice, we choose a small $L$ based on the number of devices and the number of heavy operators in the graph. We find different choices of $L$ do not affect the final performance significantly.

Alg.~\ref{alg:inter-op} summarizes the workflow of the inter-op pass and illustrates its interactions with the intra-op pass in \S\ref{sec:intra-op}.

\section{Parallelism Orchestration}
\label{sec:parallelism-orchestration}
After stages, device meshes, and their assignments are decided, at the intra-op level, \sys compiles each stage against its assigned device mesh, respecting the intra-op parallelism plan output by the ILP solver. The compilation depends on XLA~\cite{xla2017google} and GSPMD~\cite{xu2021gspmd}, and generates parallel executables for each stage-mesh pair. When needed, the compilation automatically inserts collective communication primitives (see \S\ref{sec:intra-op}) to address the \emph{within-mesh} communication caused by intra-op parallelism.

At the inter-op level, \sys implements an additional parallelism orchestration pass to address the \emph{cross-mesh} communication between stages, and generate static instructions for inter-op parallel execution. 

\begin{figure}
	\centering
	\includegraphics[width=.95\columnwidth]{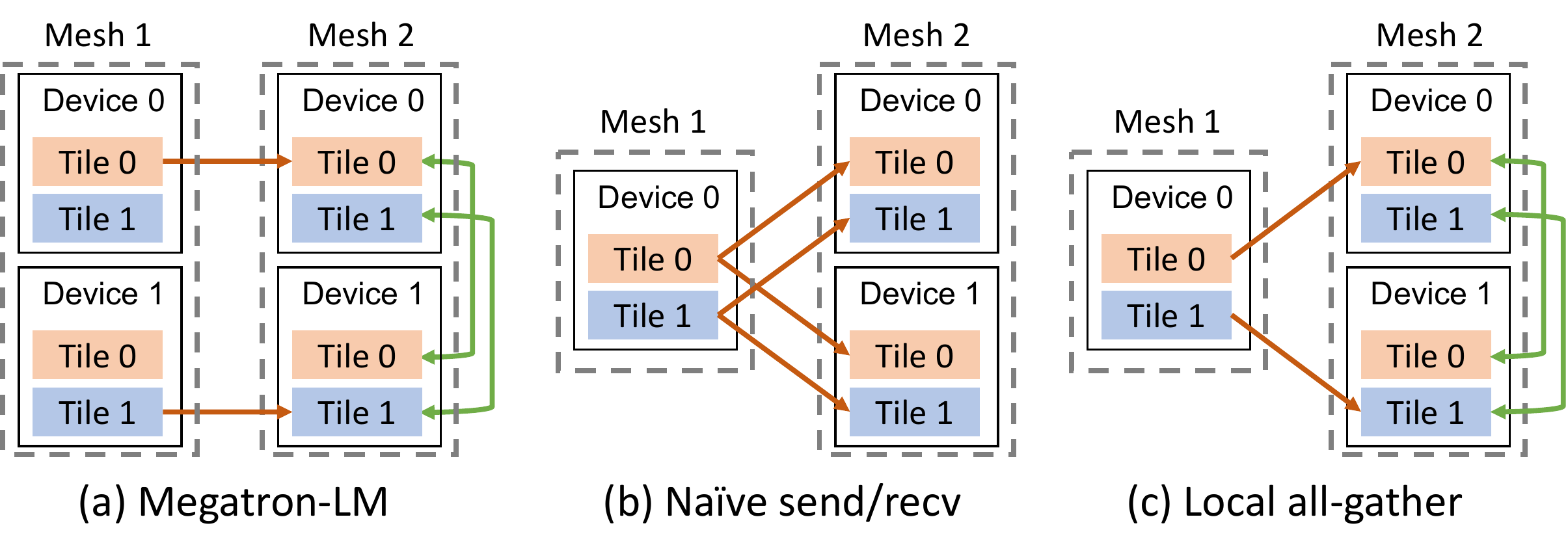}
	\vskip -0.5em
	\caption{Cross-mesh resharding. Red arrows denote send/recv on slow connections. Green arrows denote all-gather on fast connections. (a) The scatter-gather optimization for equal mesh shapes in Megatron-LM. (b) The naive send/recv for unequal mesh shapes. (c) The generalized local all-gather optimization for unequal mesh shapes.}
	\vskip -0.5em
	\label{fig:cross_mesh_resharding}
\end{figure}

\noindent \textbf{Cross-mesh resharding.}
Existing manual systems, such as Megatron-LM~\cite{rasley2020deepspeed,shoeybi2019megatron}, 
constrain all pipeline stages to have the same degrees of data and tensor model parallelism, so the communication between pipeline stages is trivially realized by P2P send/recv between corresponded devices of two equivalent device meshes (Fig.~\ref{fig:cross_mesh_resharding}a). In \sys, the device meshes holding two adjacent stages might have different mesh shapes, and the tensor to communicate between two stages might have different sharding specs (Fig.~\ref{fig:cross_mesh_resharding}b and Fig.~\ref{fig:cross_mesh_resharding}c). We call this communication pattern as \emph{cross-mesh resharding}, which is a many-to-many multicast problem.

Given the sharding specs of the tensor on the sender and receiver mesh, \sys generates a communication plan to address cross-mesh sharding in two iterations. 
In the first iteration, \sys calculates the correspondences between tensor partitions (a.k.a. tiles) on the source and destination mesh. Based on that, it generates P2P send/recv primitives between the source devices and destination devices to fulfill the communication.
It then takes a second iteration to identify opportunities where the destination tensor has a replication in its sharding spec. In this case, the tensor only needs to be transferred once between two meshes, then exchanged via all-gather across the devices on the destination mesh using its higher bandwidth (Fig.~\ref{fig:cross_mesh_resharding}) -- it rewrites send/recv generated at the first iteration into all-gather to avoid repeated communication. 

We call this approach as \emph{local all-gather} cross-mesh resharding.
Since the communication between stages is normally small by our design,  our experiments show that it performs satisfactorily well (\S\ref{sec:eval:cross-mesh-resharding}). We defer the development of the optimal cross-mesh resharding plan to future work.

\noindent \textbf{Generating pipeline execution instructions.}
As the final step, \sys generates static execution instructions to launch the training on clusters.
Since each stage has different sets of operators and may locate on meshes with different shapes, in contrast to many SPMD pipeline-parallel training systems~\cite{xu2021gspmd,narayanan2021efficient}, \sys adopts an MPMD-style runtime to orchestrate the inter-op parallel execution -- \sys generates distinct static execution instructions for each device mesh.

\sys develops a set of instructions for inter-op parallel execution, including instructions for allocating and deallocating memory for tensors in a stage, communicating tensors between stages following the cross-mesh resharding plan, synchronization, and computation, etc. According to a user-selected pipeline schedule, \sys uses a driver process to generate the instructions in advance and dispatches the whole instruction lists to each worker before execution, avoiding driver-worker coordination overheads during runtime.

\section{Limitations and Discussion}
In this section, we discuss advantages of our view of parallelisms and several limitations of our algorithms.

Compared to existing work that manually combines data, operator, and pipeline parallelism, such as 3D parallelism~\cite{rasley2020deepspeed} and PTD-P~\cite{narayanan2021efficient}, Alpa's hierarchical view of inter- and intra-op parallelisms significantly advances them with three major flexibility: 
(1) pipeline stages can contain an uneven number of operators or layers; 
(2) pipeline stages in Alpa might be mapped to device meshes with different shapes; 
(3) within each stage, the data and operator parallelism configuration is customized non-uniformly on an operator-by-operator basis.
Together, they allow Alpa to unify all existing model parallelism approaches and generalize to model architectures and cluster setups with more heterogeneity.

Despite these advantages, Alpa's optimization algorithms currently have a few limitations:

\noindent \(\bullet\) 
Alpa does not model the communication cost between different stages
because the cross-stage communication cost is \emph{by nature small}. In fact, modeling the cost in either the DP or ILP is possible, but would require enumerating exponentially more intra-op passes and DP states.

\noindent \(\bullet\) The inter-op pass currently has a hyperparameter: the number of micro-batches $B$, which is not optimized by our current formulation but can be searched by enumeration.

\noindent \(\bullet\) 
The inter-op pass models pipeline parallelism with a static linear schedule, without considering more dynamic schedules that, for example, parallelize different branches in a computational graph on different devices.

\noindent \(\bullet\) Alpa does not optimize for the best scheme of overlapping computation and communication; Alpa can only handle static computational graphs with all tensor shapes known at compilation time.

Nevertheless, our results on weak scaling (\S\ref{sec:evaluation}) suggest that \sys is able to generate near-optimal execution plans for many notable models.


\section{Evaluation}
\label{sec:evaluation}
\begin{figure*}
	\centering
	\begin{subfigure}[b]{0.33\textwidth}
		\centering
		\includegraphics[width=\textwidth]{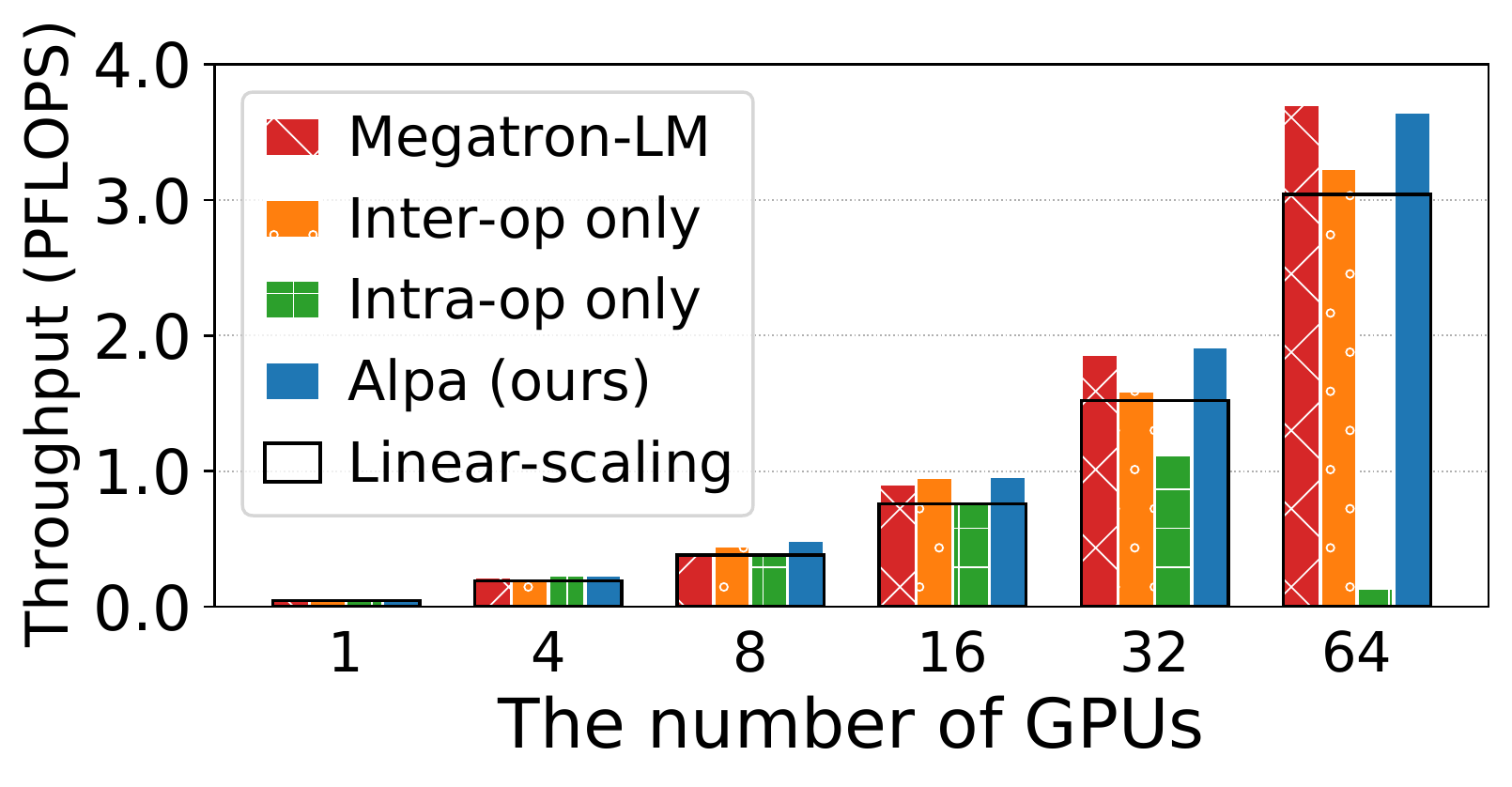}
		\vskip -0.6em
		\caption{GPT}
		\label{fig:gpt-end-to-end-result}
	\end{subfigure}
	\begin{subfigure}[b]{0.33\textwidth}
		\centering
		\includegraphics[width=\textwidth]{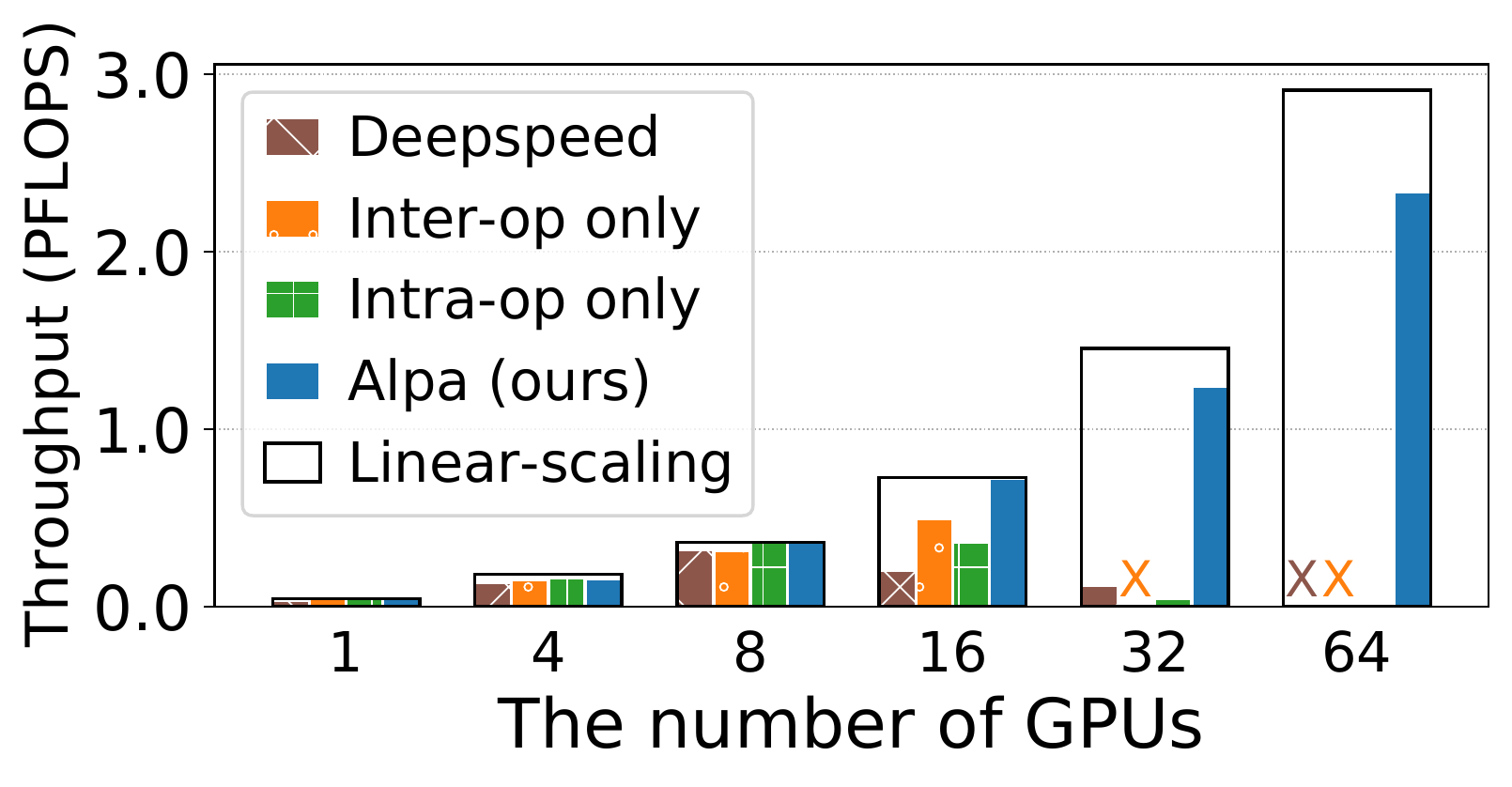}
		\vskip -0.6em
		\caption{MoE}
	\end{subfigure}
	\begin{subfigure}[b]{0.33\textwidth}
		\centering
		\includegraphics[width=\textwidth]{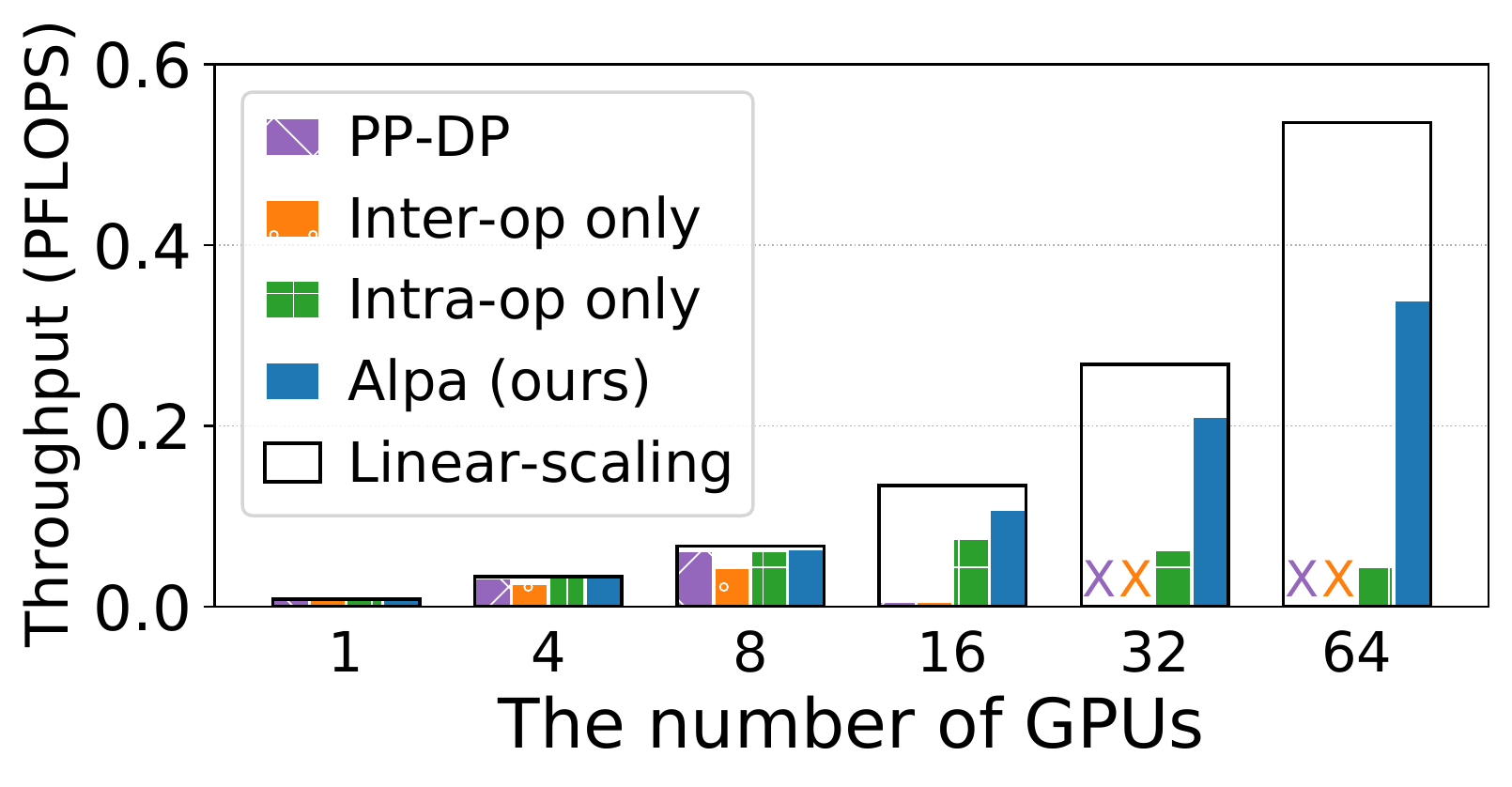}
		\vskip -0.6em
		\caption{Wide-ResNet}
	\end{subfigure}
	\vskip -0.8em
	\caption{End-to-end evaluation results. ``$\times$'' denotes out-of-memory. Black boxes represent linear scaling.}
	\vskip -1em
	\label{fig:evaluation:e2e}
\end{figure*}

\sys is implemented using about 16K LoC in Python and 6K LoC in C++.
\sys uses Jax as the frontend and XLA as the backend. The compiler passes are implemented on Jax's and XLA's intermediate representation (i.e., Jaxpr and HLO).
For the distributed runtime, we use Ray \cite{moritz2018ray} actor to implement the device mesh worker, XLA runtime for executing computation, and NCCL \cite{nccl} for communication.

We evaluate \sys on training large-scale models with billions of parameters, including GPT-3~\cite{brown2020language}, GShard Mixture-of-Experts (MoE)~\cite{lepikhin2020gshard}, and Wide-ResNet~\cite{zagoruyko2016wide}.
The testbed is a typical cluster consisting of 8 nodes and 64 GPUs.
Each node is an Amazon EC2 p3.16xlarge instance with 8 NVIDIA V100 16 GB GPUs, 64 vCPUs, and 488 GB memory.
The 8 GPUs in a node are connected via NVLink.
The 8 nodes are launched within one placement group with 25Gbps cross-node bandwidth.

We compare \sys against two state-of-the-art distributed systems for training large-scale models on GPUs.
We then isolate different compilation passes and perform ablation studies of our optimization algorithms.
We also include a case study of the execution plans found by \sys.

\subsection{End-to-End Performance}
\label{subsec:e2e}

\begin{table}[t]
\centering
\footnotesize
\vskip -0.5em
\caption{Models used in the end-to-end evaluation. LM = language model. IC = image classification.}
\vskip -1em

\scalebox{0.85}{
\begin{tabular}{ccccc}
\toprule
Model & Task & Batch size & \#params (billion) & Precision  \\
\midrule
GPT-3 \cite{brown2020language} & LM &  1024 & 0.35, 1.3, 2.6, 6.7, 15, 39 & FP16\\
GShard MoE \cite{lepikhin2020gshard} & LM & 1024 & 0.38, 1.3, 2.4, 10, 27, 70 & FP16\\
Wide-ResNet \cite{zagoruyko2016wide} & IC &  1536 & 0.25, 1.0, 2.0, 4.0, 6.7, 13 & FP32\\
\bottomrule
\end{tabular}
}
\vskip -1.5em
\label{table:models}
\end{table}

\noindent \textbf{Models and training workloads.} We target three types of models listed in Table~\ref{table:models}, covering models with both homogeneous and heterogeneous architectures. GPT-3 is a homogeneous transformer-based LM by stacking many transformer layers whose model parallelization plan has been extensively studied~\cite{shoeybi2019megatron, narayanan2021efficient}.
GShard MoE is a mixed dense and sparse LM, where mixture-of-experts layers are used to replace the MLP at the end of a transformer, every two layers.
Wide-ResNet is a variant of ResNet with larger channel sizes. It is vastly different from the transformer models and there are no existing manually designed strategies.

To study the ability to train large models, we follow common ML practice to scale the model size along with the number of GPUs, with the parameter range reported in Table~\ref{table:models}.
More precisely, for GPT-3, we increase the hidden size and the number of layers together with the number of GPUs following~\cite{narayanan2021efficient}, whereas for MoE we mainly increase the number of experts suggested by~\cite{lepikhin2020gshard,xu2021gspmd}.  For Wide-ResNet, we increase the channel size and width factor in convolution layers.
For each model, we adopt the suggested global batch size per ML practice~\cite{narayanan2021efficient,brown2020language,lepikhin2020gshard,zagoruyko2016wide} to keep the same statistical behavior.
We then tune the best microbatch size for each model and system configuration that maximizes the system performance. The gradients are accumulated across microbatches.
The detailed model specifications are provided in Appendix~\ref{sec:model_spec}.

\noindent \textbf{Baselines.} 
For each model, we compare \sys against a strong baseline.
We use Megatron-LM v2~\cite{narayanan2021efficient} as the baseline system for GPT-3.
Megatron-LM is the state-of-the-art system for training homogeneous transformer-based LMs on GPUs.
It combines data parallelism, pipeline parallelism, and manually-designed operator parallelism (denoted as TMP later).
The combination of these techniques is controlled by three integer parameters that specify the parallelism degrees assigned to each technique.
We grid-search the three parameters following the guidance of their paper and report the results of the best configuration. Megatron-LM is specialized for GPT-like models, so it does not support other models in Table~\ref{table:models}.

We use DeepSpeed \cite{rasley2020deepspeed} as the baseline for MoE. DeepSpeed provides a state-of-the-art implementation for training MoE on GPUs.
It combines handcrafted operator parallelism for MoE layers and ZeRO-based~\cite{rajbhandari2020zero} data parallelism. The combination of these techniques is controlled by several integer parameters that specify the parallelism degree assigned to each technique. We also grid-search them and report the best results.
The performance of DeepSpeed on \mbox{GPT-3} is similar to or worse than Megatron-LM, so we skip it on GPT-3. Note that original GShard-MoE~\cite{lepikhin2020gshard} implementation is only available on TPUs, thus we do not include its results, though their strategies~\cite{lepikhin2020gshard} are covered by \sys’s strategy space.

For large Wide-ResNet, there is no specialized system or manually designed plan for it. We use \sys to build a baseline ``PP-DP'' whose space only consists of data parallelism and pipeline parallelism, which mimics the parallelism space of PipeDream \cite{narayanan2019pipedream} and Dapple \cite{fan2021dapple}.

For all models, we also include the results of using \sys with only one of intra- and inter-operator parallelism, which mimics the performance of some other auto-parallel systems. The open-source Flexflow~\cite{jia2018beyond} does not support the models we evaluate, as it lacks support for many necessary operators (e.g., layer normalization~\cite{ba2016layer}, mixed-precision operators). Tofu~\cite{wang2019supporting} only supports single node execution and is not open-sourced. Due to both theoretical and practical limitations, we do not include their results and we do not expect Flexflow or Tofu to outperform the state-of-the-art manual baselines in our evaluation.

\noindent \textbf{Evaluation metrics.}
\sys does not modify the semantics of the synchronous gradient descent algorithm, thus does not affect the model convergence. Therefore, we measure training throughput in our evaluation. We evaluate weak scaling of the system when increasing the model size along with the number of GPUs.
Following~\cite{narayanan2021efficient}, we use the aggregated peta floating-point operations per second (PFLOPS) of the whole cluster as the metric\footnote{As the models are different for different numbers of GPUs, we cannot measure scaling on the system throughput such as tokens per second or images per second.}. We measure it by running a few batches with dummy data after proper warmup. All our results (including those in later sections) have a standard deviation within $0.5\%$, so we skip the error bars in our figures.

\noindent \textbf{GPT-3 results.}
\label{sec:GPT-result}
The parallelization plan for GPT-3 has been extensively studied~\cite{brown2020language,narayanan2021efficient,li2021terapipe}. 
We observe in Fig.~\ref{fig:evaluation:e2e}a that this manual plan with the best grid-searched parameters enables Megatron-LM to achieve super-linear weak scaling on GPT-3. Nevertheless, compared to Megatron-LM, \sys automatically generates execution plans and even achieves slightly better scaling on several settings. 
If compared to methods that only use intra-operator parallelism, our results are consistent with recent studies -- ``Intra-op only'' performs poorly on >16 GPUs because even the best plan has to communicate tensors heavily on cross-node connections, making communication a bottleneck. Surprisingly, ``Inter-op only'' performs well and maintains linear scaling on up to 64 GPUs. 

We investigate the grid-searched parameters of the manual plan on Megatron-LM, and compare it to the plan generated by \sys. It reveals two major findings. 
First, in Megatron-LM, the best manual plan has TMP as 1, except in rare settings, such as fitting the 39B model on 64 GPUs, where pipeline parallelism alone is unable to fit the model (stage) in GPU memory; meanwhile, data parallelism is maximized whenever memory allows. 
In practice, gradient accumulation (GA) is turned on to achieve a desired global batch size (e.g., 1024 in our setting). GA amortizes the communication of data parallelism and reduces the bubbles of pipeline parallelism, but the communication of TMP  grows linearly with GA steps, which puts TMP disadvantaged.
Second, \sys-generated plan closely resembles the best-performed ones in Megatron-LM, featuring (1) evenly-sized stages, (2) partitioning along the batch dimension in stages when memory is not stressed, but along non-batch dimensions when memory is stressed. One key difference between our plan and the manual plan is that \sys also partitions the weight update operations when data parallelism exists, which contributes to the slight performance improvement over Megatron-LM. This attributes to the fact that \sys, as a generic compiler system, can compose a wide range of parallelism approaches, while Megatron-LM, for now, misses weight update sharding support.

\begin{figure*}[t]
	\centering
	\begin{subfigure}[b]{0.33\textwidth}
		\centering
		\includegraphics[width=\textwidth]{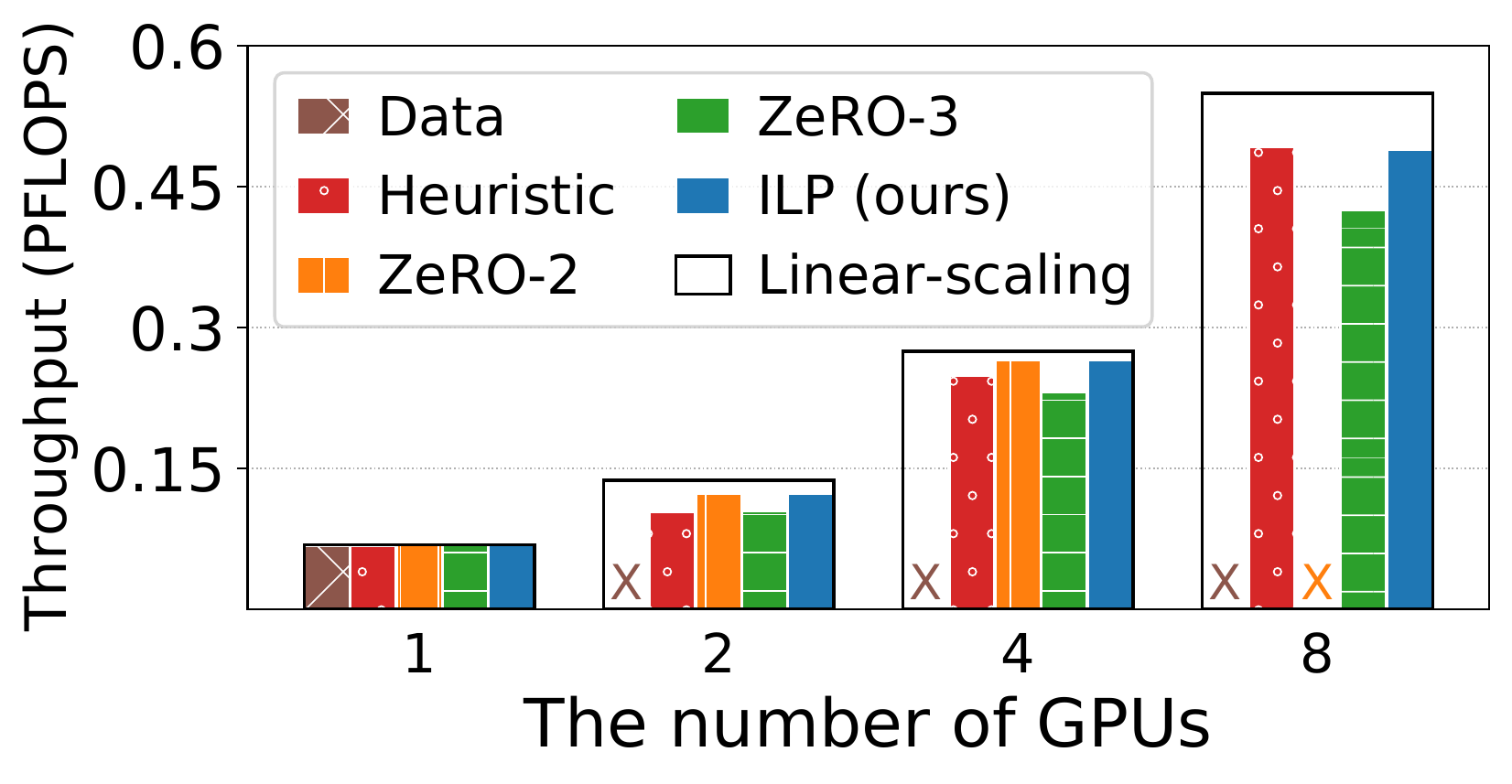}
		\vskip -0.6em
		\caption{GPT}
	\end{subfigure}
	\begin{subfigure}[b]{0.33\textwidth}
		\centering
		\includegraphics[width=\textwidth]{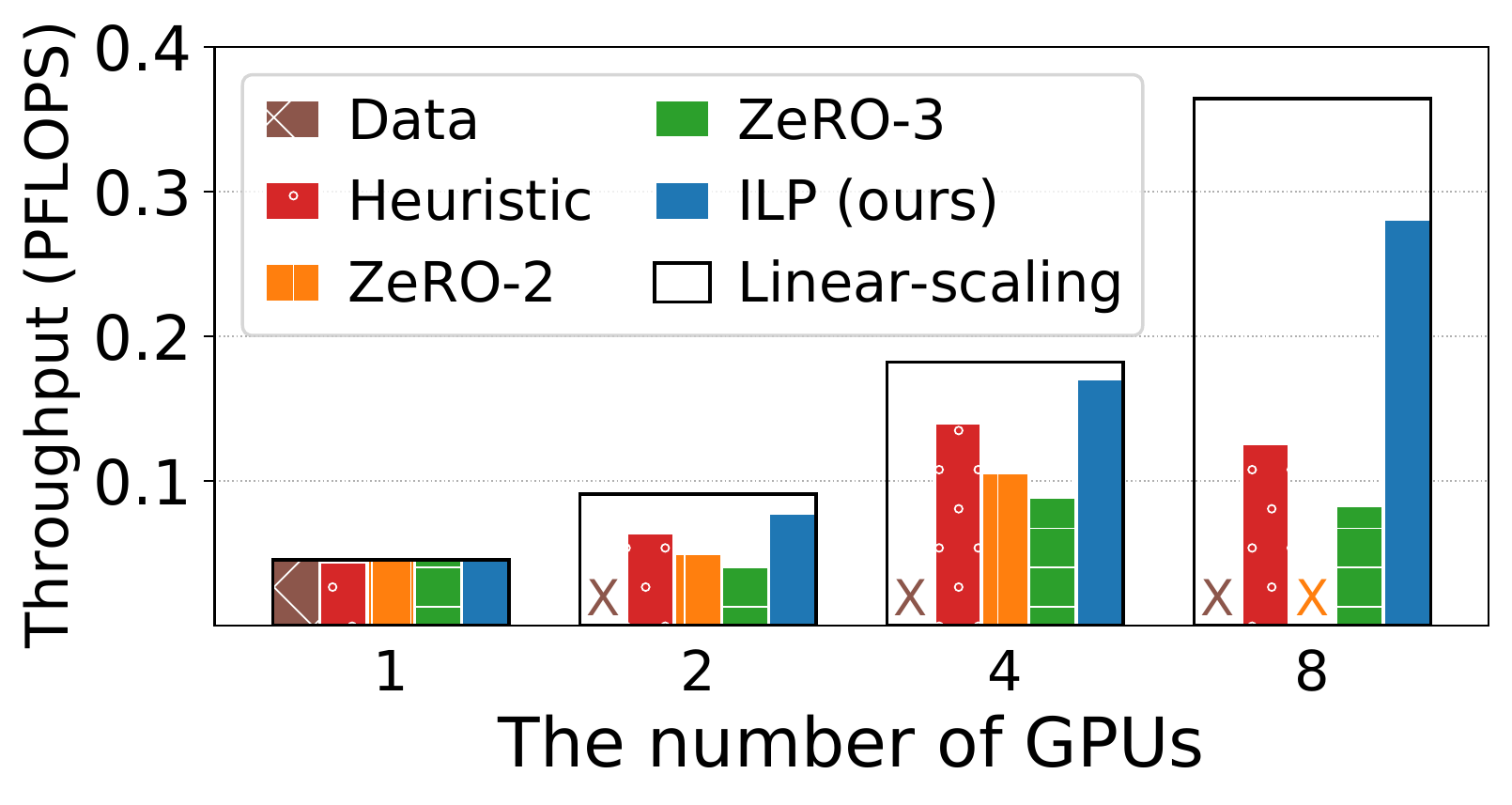}
		\vskip -0.6em
		\caption{MoE}
	\end{subfigure}
	\begin{subfigure}[b]{0.33\textwidth}
		\centering
		\includegraphics[width=\textwidth]{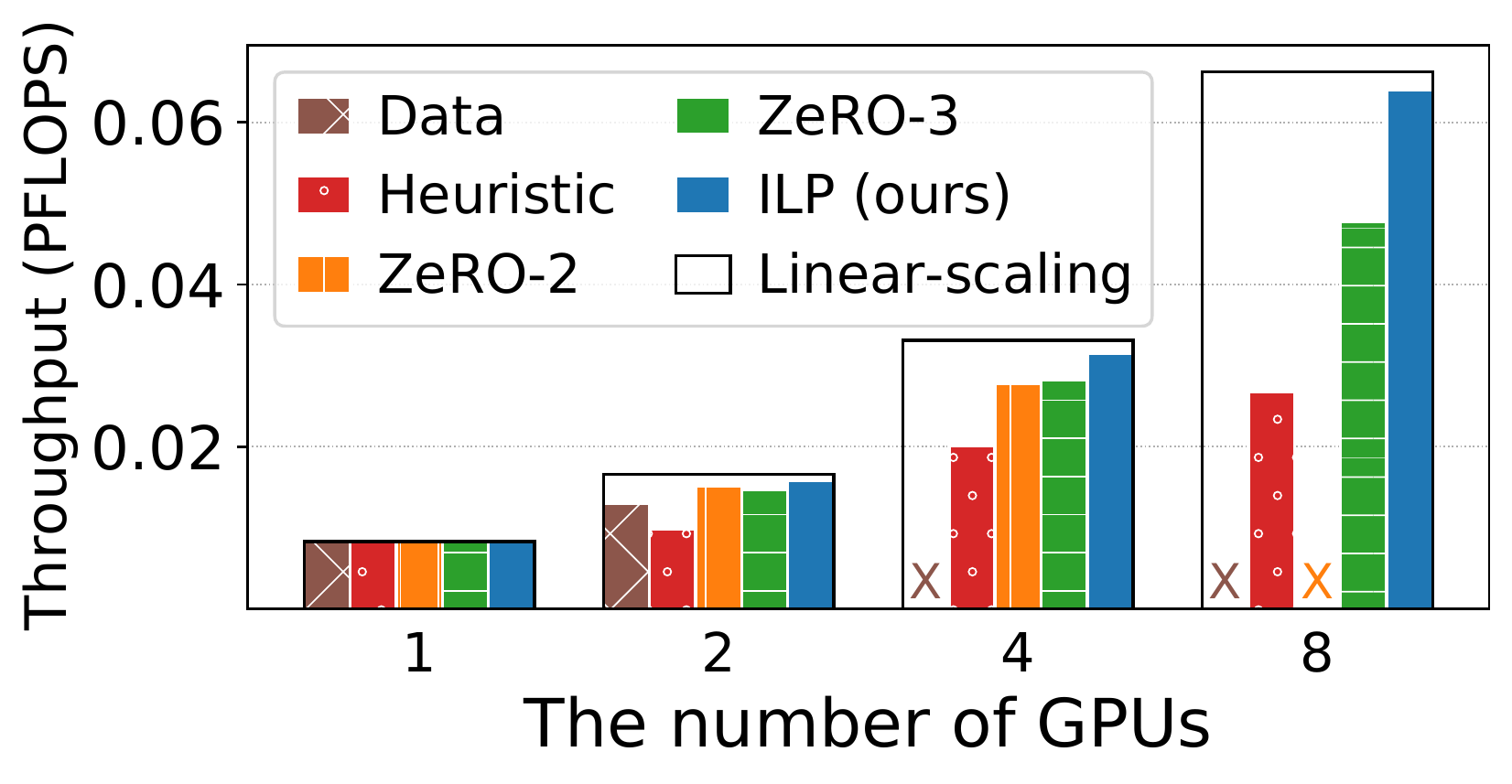}
		\vskip -0.6em
		\caption{Wide-ResNet}
	\end{subfigure}
	\vskip -1em
	\caption{Intra-operator parallelism ablation study. ``$\times$'' denotes out-of-memory. Black boxes represent linear scaling.}
	\vskip -1em
	\label{fig:evaluation:intra-op}
\end{figure*}

\noindent \textbf{MoE results.}
DeepSpeed adopts a manual operator parallelism plan for MoE models, developed by GShard~\cite{lepikhin2020gshard}, called \emph{expert parallelism}, which uses a simple rule: it partitions the expert axis for the operators in MoE layers, but switches back to data parallelism for non-expert layers.
This {expert parallelism} is then combined with ZeRO data parallelism and TMP.
All of these techniques belong to intra-operator parallelism.
Unfortunately, DeepSpeed's specialized implementation does not include any inter-operator parallelism approach, which is required for scaling across multiple nodes with low inter-node bandwidth.
Therefore, Deepspeed only maintains a good performance within a node ($\le 8$ GPUs) on this cluster.
``Intra-op only'' fails to scale across multiple nodes due to the same reason.
``Inter-op only'' runs out of memory on 32 GPUs and 64 GPUs because it is not easy to equally slice the model when the number of GPUs is larger than the number of layers of the model. The imbalanced slicing makes some memory-intensive stages run out of memory.

By contrast, \sys automatically discovers the best execution plans that combine intra- and inter-operator parallelism.
For intra-operator parallelism, \sys finds a strategy similar to expert parallelism and combines it with ZeRO data parallelism, thanks to its ILP-based intra-op pass.
\sys then constructs stages and uses inter-operator parallelism to favor small communication volume on slow connections.
\sys maintains linear scaling on 16 GPUs and scales well to 64 GPUs. Compared to DeepSpeed, \sys achieves $3.5\times$ speedup on 2 nodes and a $9.7\times$ speedup on 4 nodes.

\noindent \textbf{Wide-ResNet results.}
Unlike the previous two models that stack the same layer, Wide-ResNet has a more heterogeneous architecture. As the data batch is forwarded through layers, the size of the activation tensor shrinks while the size of the weight tensor inflates. This leads to an imbalanced distribution of memory usage and compute intensity across layers.
For this kind of model, it is difficult, if not impossible, to manually design a plan. However, \sys still achieves a scalable performance on 32 GPUs with 80\% scaling. The baselines ``PP-DP'' and ``Inter-op only'' run out of memory when training large models, because they cannot partition weights to reduce the memory usage, and it is difficult to construct memory-balanced stages for them. ``Intra-only''  requires a lot of communication on slow connections, so it cannot scale across multiple nodes. A case study on the generated plan for Wide-ResNet is in \S\ref{sec:eval:case-study}.

\subsection{Intra-Op Parallelism Ablation Study}
We study the effectiveness of our intra-operator parallelism optimization algorithm.
We compare our ILP-based solution against alternatives such as ZeRO optimizer and rule-based partitioning strategies.
 
 \noindent \textbf{Experimental setup.}
We run a weak scaling benchmark in terms of model size similar to \S\ref{subsec:e2e}, but disable pipeline parallelism and gradient accumulation to control variables. The benchmark is done on one AWS p3.16xlarge instance with 8 GPUs. In order to simulate an execution environment of large-scale training in one node, we use larger hidden sizes, smaller batch sizes, and smaller numbers of layers, compared to the model configurations in \S\ref{subsec:e2e}.

 \noindent \textbf{Baselines.}
 We compare automatic solutions for intra-operator parallelism.
``Data'' is vanilla data parallelism. 
``ZeRO-2''~\cite{rajbhandari2020zero} is a memory-efficient version of data parallelism which partitions gradients and optimizer states. 
``ZeRO-3''~\cite{rajbhandari2020zero} additionally partitions parameters on top of ``ZeRO-2''. 
``Heuristic'' uses a rule combined with the sharding propagation in GSPMD. It marks the largest dimension of every input tensor as partitioned and runs sharding propagation to get the sharding specs for all nodes in the graph.
``ILP'' is our solution based on the ILP solver.

\noindent \textbf{Results.}
As shown in Fig.~\ref{fig:evaluation:intra-op}, ``Data'' runs out of memory quickly and cannot train large models. ``ZeRO-2'' and ``ZeRO-3'' resolve the memory problem of data parallelism, but they do not optimize for communication as they always communicate the gradients. When the gradients are much larger than activations, their performance degenerates. ``Heuristic'' solves the memory issue by partitioning all tensors, but can be slowed down by larger communication. ``Auto-sharding'' performs best in all cases and maintains a near-linear scaling, because it figures out the correct partition plan that always minimizes the communication overhead.

\begin{figure}
	\vskip -1em

	\centering
	\begin{subfigure}[b]{\columnwidth}
	\centering
	\includegraphics[width=0.6\textwidth]{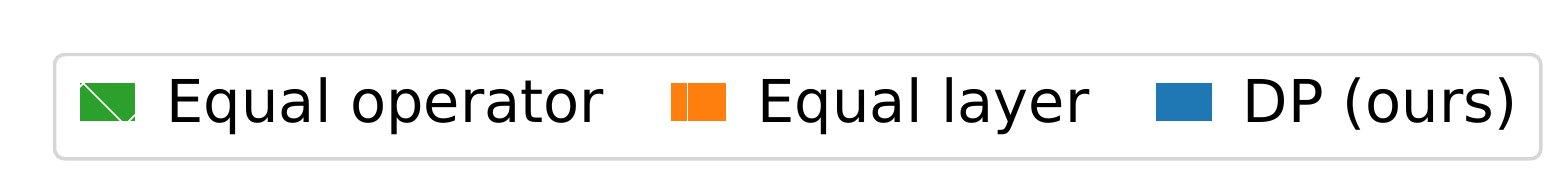}
    \end{subfigure}

	\begin{subfigure}[b]{0.30\columnwidth}
		\centering
		\includegraphics[width=\textwidth]{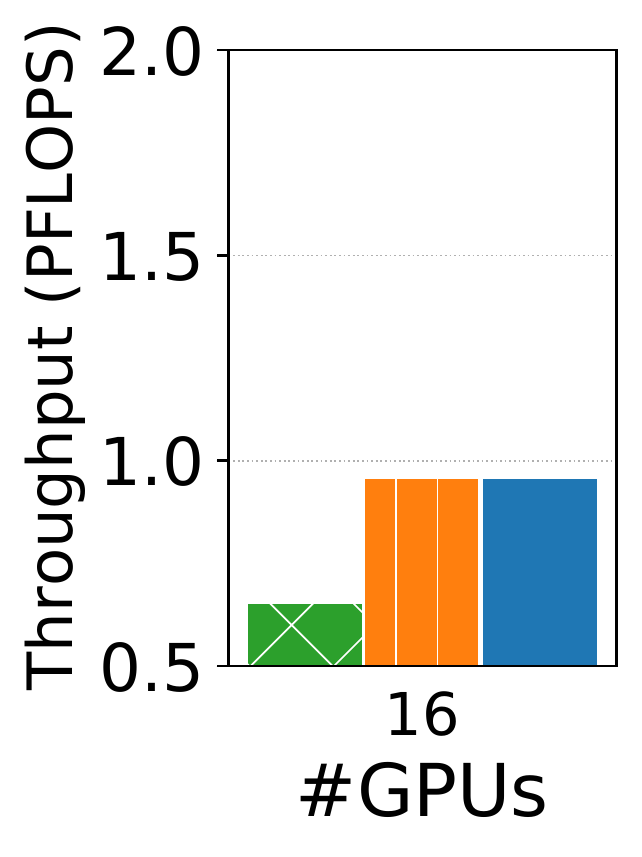}
		\vskip -0.6em
		\caption{GPT}
	\end{subfigure}
	\begin{subfigure}[b]{0.50\columnwidth}
		\centering
		\includegraphics[width=\textwidth]{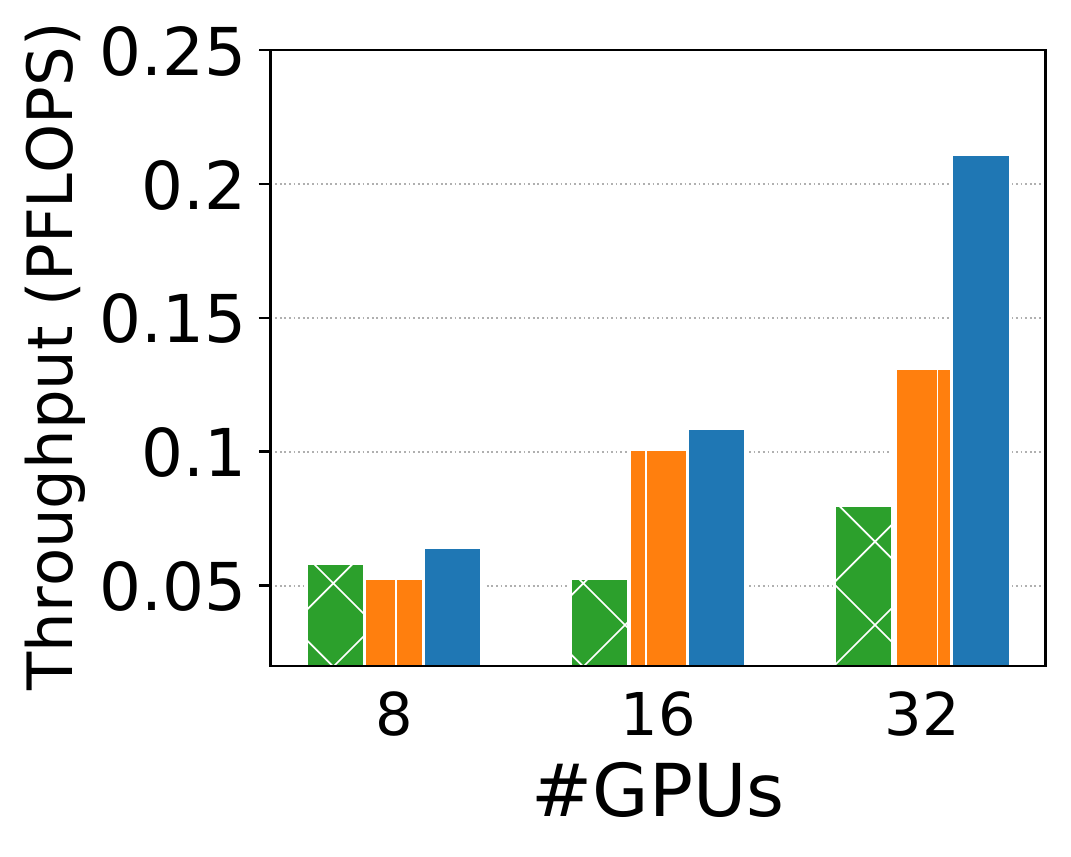}
		\vskip -0.6em
		\caption{Wide-ResNet}
	\end{subfigure}
	\vskip -1em
	\caption{Inter-operator parallelism ablation study.}
	\vskip -0.5em
	\label{fig:evaluation:inter-op}
\end{figure}

\subsection{Inter-Op Parallelism Ablation Study}
We study the effectiveness of our inter-operator parallelism optimization algorithm. We use ``DP'' to denote our algorithm.

\noindent \textbf{Experimental setup.}
We report the performance of three variants of our DP algorithm on GPT and Wide-ResNet. The benchmark settings are the same as the settings in \S\ref{subsec:e2e}.

\noindent \textbf{Baselines.}
We compare our DP algorithm with two rule-based algorithms.
``Equal operator'' disables our DP-based operator clustering but assigns the same number of operators to each cluster.
``Equal layer'' restricts our DP algorithm to use the same number of layers for all stages.

\noindent \textbf{Results.}
Fig.~\ref{fig:evaluation:inter-op} shows the result.
``DP'' always outperforms ``Equal operator''. This is because ``Equal operator'' merges operator that should be put onto different device meshes. \sys's algorithm can cluster operators based on the communication cost and computation balance.
Whether ``DP'' can outperform ``Equal layer'' depends on the model architecture.
On homogeneous models like GPT, the solution of our DP algorithm uses the same number of layers for all stages, so ``Equal layer'' performs the same as ``DP''.
On Wide-ResNet, the optimal solution can assign different layers to different stages, so ``Equal layer'' is worse than the full flexible DP algorithm.
For Wide-ResNet on 32 GPUs, our algorithm outperforms ``Equal operator'' and ``Equal layer'' by $2.6\times$ and $1.6\times$, respectively.

\begin{figure}[t]
	\centering
	\footnotesize
	\includegraphics[width=.80\columnwidth]{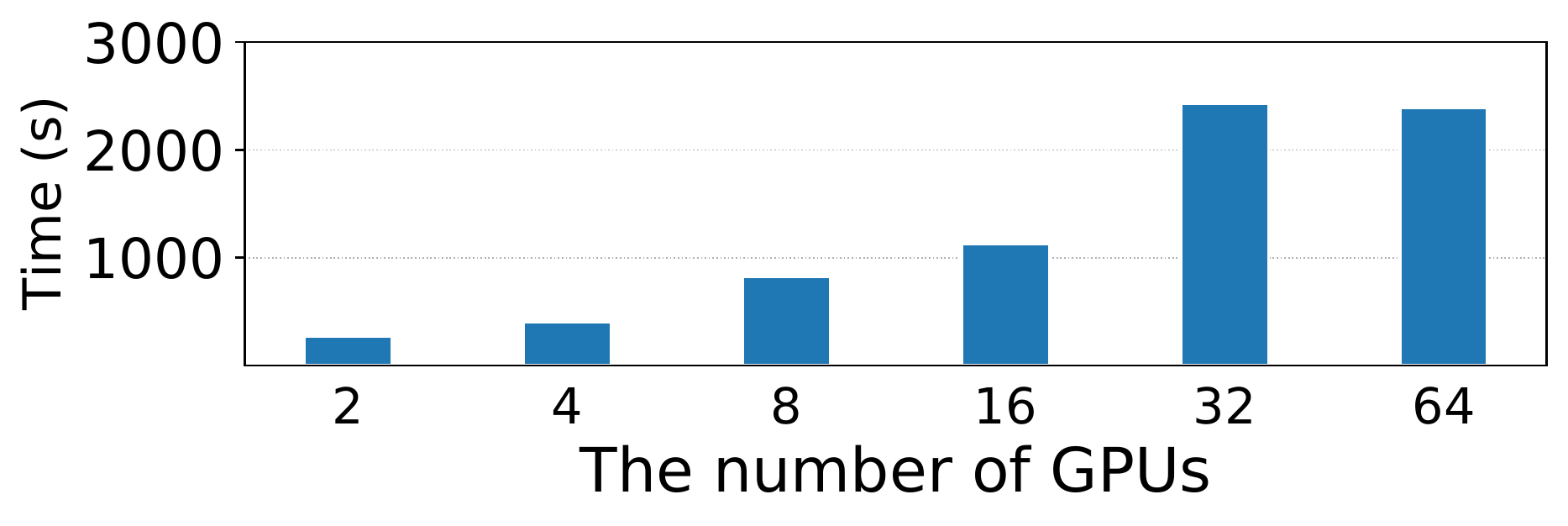}
	\vskip -0.5em
	\caption{\sys's compilation time on all GPT models. The model size and \#GPUs are simultaneously scaled.}
	\vskip -1.5em
	\label{fig:compilation-time-all}
\end{figure}

\subsection{Compilation Time}
Fig.~\ref{fig:compilation-time-all} shows \sys's compilation time for all the GPT settings in \S\ref{sec:GPT-result}. The compilation time is a single run of the full Alg.~\ref{alg:inter-op} with a provided number of microbatches $B$.
According to the result, \sys scales to large models or large clusters well, because compilation time grows linearly with the size of the model and the number of GPUs in the cluster. Table~\ref{table:compilation-time} reports the compilation time breakdown for the largest GPT model in our evaluation (39B, 64 GPUs).
Most of the time is spent on enumerating stage-mesh pairs and profiling them.
For the compilation part, we accelerate it by compiling different stages in parallel with distributed workers.
For profiling, we accelerate it using a simple cost model built at the XLA instruction level, which estimates the cost of matrix multiplication and communication primitives with a piece-wise linear model.
With these optimizations, the compilation and search for a model take at most several hours, which is acceptable as it is much shorter than the actual training time, which can take several weeks.

\subsection{Cross-Mesh Resharding}
\label{sec:eval:cross-mesh-resharding}
We evaluate our generalized local all-gather optimization for cross-mesh resharding between meshes with different shapes on Wide-ResNet, as shown in Fig.~\ref{fig:eval_cross_mesh_resharding}.
``signal send/recv'' is a synthetic case where we only send 1 signal byte between stages, which can be seen as the upper bound of the performance. ``w/o local all-gather'' disables our local all-gather optimization and uses only send/recv. ``w/ local all-gather'' enables our local all-gather optimization to move more communication from slow connections to fast local connections, which brings 2.0$\times$ speedup on 32 GPUs.  

\begin{table}[t]
	\centering
	\footnotesize
	\caption{Compilation time breakdown of GPT-39B.}
    \vskip -0.5em
    \scalebox{1.0}{
	\begin{tabular}{ccc}
		\toprule
		Steps                   & Ours       & w/o optimization \\
		\midrule
		Compilation             & 1582.66 s  & > 16hr  \\
		Profiling               & 804.48 s   & > 24hr \\
		Stage Construction DP   & 1.65 s     & N/A \\
		Other                   & 4.47 s     & N/A \\
		Total                   & 2393.26 s     & > 40hr \\
		\bottomrule
	\end{tabular}
	}
	\label{table:compilation-time}
\end{table}

\begin{figure}[t]
	\centering
	\footnotesize
	\includegraphics[width=.80\columnwidth]{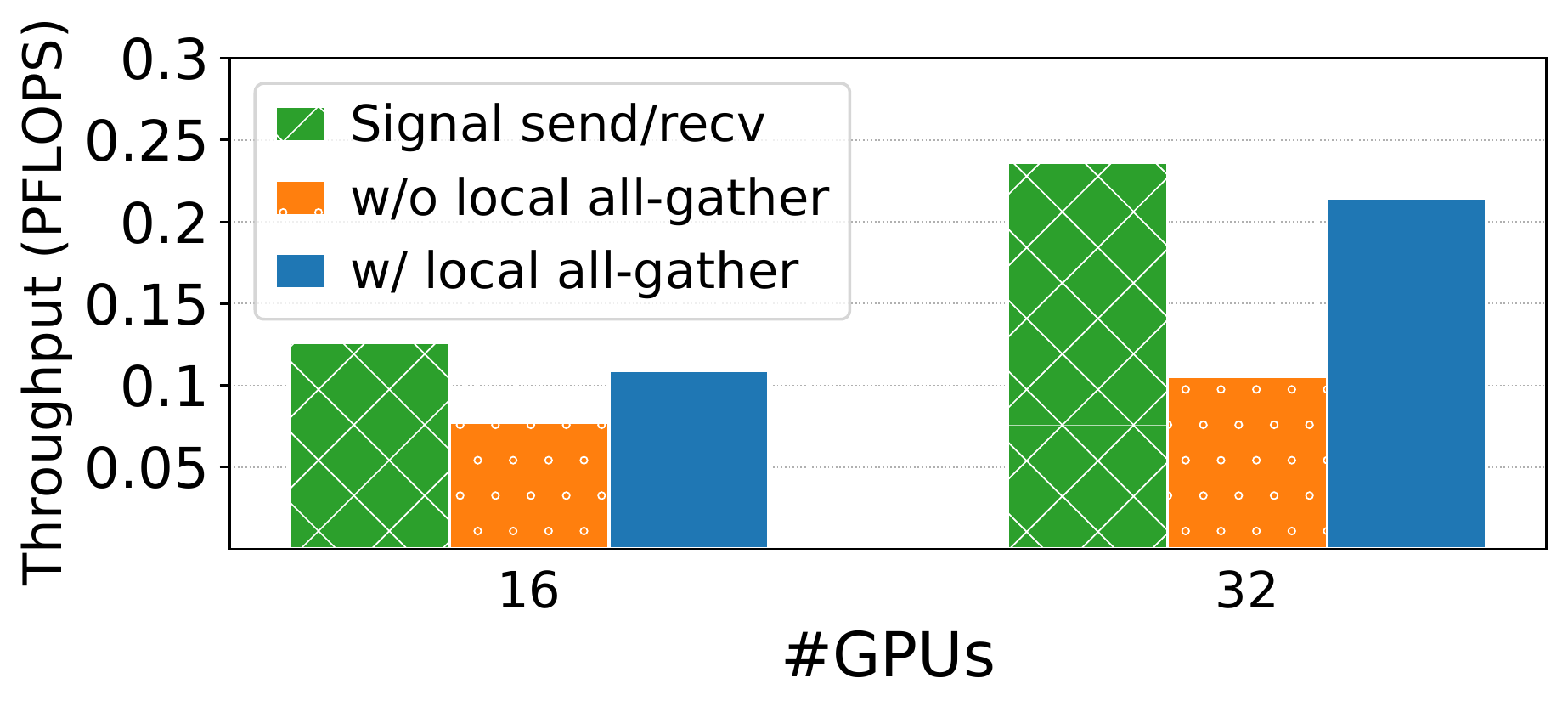}
	\vskip -1em
	\caption{Cross-mesh resharding on Wide-ResNet.}
	\vskip -1.5em
	\label{fig:eval_cross_mesh_resharding}
\end{figure}

\begin{figure*}
	\centering
    \includegraphics[width=.95\textwidth]{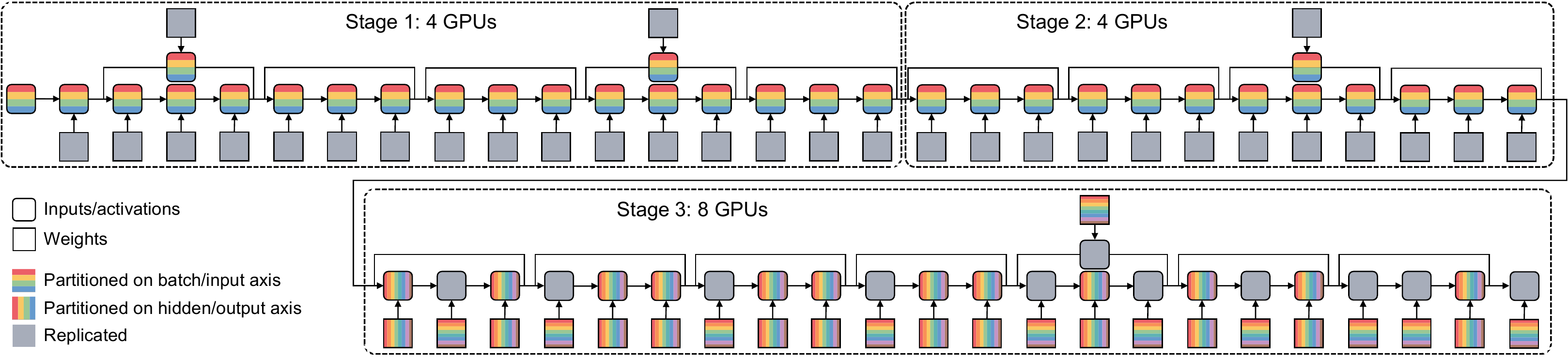}
	\vskip -0.5em
	\caption{Visualization of the parallel strategy of Wide-ResNet on 16 GPUs. Different colors represent the devices a tensor is distributed on. Grey blocks indicate a tensor is replicated across the devices. The input data and resulting activation of each convolution and dense layer can be partitioned along the batch axis and the hidden axis. The weights can be partitioned along the input and output channel axis.}
	\vskip -1em
	\label{fig:evaluation:case_wresnet}
\end{figure*}

\subsection{Case Study: Wide-ResNet}
\label{sec:eval:case-study}
We visualize the parallelization strategies \sys finds for Wide-ResNet on 16 GPUs in Fig.~\ref{fig:evaluation:case_wresnet}. We also include the visualization of results on 4 and 8 GPUs in Appendix~\ref{sec:extra_case_study}. 
On 4 GPUs, \sys uses only intra-operator parallelism. The intra-operator solution partitions along the batch axis for the first dozens of layers and then switches to partitioning the channel axis for the last few layers.
On 16 GPUs, \sys slices the model into 3 stages and assigns 4, 4, 8 GPUs to stage 1, 2, 3, respectively.
Data parallelism is preferred in the first two stages because the activation tensors are larger than weight tensors.
In the third stage, the ILP solver finds a non-trivial way of partitioning the convolution operators. The result shows that it can be opaque to manually create such a strategy for a heterogeneous model like Wide-ResNet, even for domain experts.

\section{Related Work}
\noindent \textbf{Systems for data-parallel training.} 
Horovod~\cite{sergeev2018horovod} and PyTorchDDP~\cite{li2020pytorchddp} are two commonly adopted data-parallel training systems that synchronize gradients using all-reduce.
BytePS~\cite{peng2019generic,jiang2020unified} unifies all-reduce and parameter servers and utilizes heterogeneous resources in data center clusters. 
AutoDist~\cite{zhang2020autosync} uses learning-based approaches to compose a data-parallel training strategy.
ZeRO~\cite{rajbhandari2020zero,xu2020automatic} improves the memory usage of data parallelism by reducing replicated tensors. 
MiCS~\cite{zhang2022mics} minimizes the communication scale on top of ZeRO for better scalability on the public cloud.
In \sys, data parallelism~\cite{kim2019parallax} reduces to a special case of intra-operator parallelism -- partitioned along the batch axis.

\noindent \textbf{Systems for model-parallel training.} 
The two major classes of model parallelisms have been discussed in \S\ref{sec:background}. Mesh-TensorFlow~\cite{shazeer2018mesh}, GSPMD~\cite{xu2021gspmd,lepikhin2020gshard} and OneFlow~\cite{yuan2021oneflow} provide annotation APIs for users to manualy specifiy the intra-op parallel plan.
ColocRL~\cite{mirhoseini2017device} puts disjoint model partitions on different devices \emph{without pipelining}, thereby the concurrency happens only when there exist parallel branches in the model.
In contrast, Gpipe~\cite{huang2019gpipe} splits the input data into micro-batches and forms pipeline parallelisms.
PipeDream~\cite{narayanan2019pipedream,narayanan2021memory} improves GPipe by using asynchronous training algorithms, reducing memory usage, and integrating it with data parallelism. However, PipeDream is asynchronous while \sys is a synchronous training system.
TeraPipe~\cite{li2021terapipe} discovers a new pipeline parallelism dimension for transformer-based LMs.
Google's Pathway system~\cite{barham2022pathways} is a concurrent work of \sys. Pathway advocates a single controller runtime architecture combining "single program multiple data" (SPMD) and "multiple program multiple data" (MPMD) model.
This is similar to \sys's runtime part, where SPMD is used for intra-op parallelisms and MPMD is used for inter-op parallelism.


\noindent \textbf{Automatic search for model-parallel plans.} Another line of work focuses on the automatic discovery of model-parallel training plans.
Tofu~\cite{wang2019supporting} develops a dynamic programming algorithm to generate the optimal intra-op strategy for \emph{linear} graphs on \emph{a single node}.
FlexFlow~\cite{jia2018beyond} proposes a ``SOAP'' formulation and develops an MCMC-based randomized search algorithm.
However, it only supports device placement without pipeline parallelism. Its search algorithm cannot scale to large graphs or clusters and does not have optimality guarantees. TensorOpt~\cite{cai2021tensoropt} develops a dynamic programming algorithm to automatically search for intra-op strategies that consider both memory and computation cost. Varuna~\cite{athlur2022varuna} targets low-bandwidth clusters and focuses on automating pipeline and data parallelism. Piper~\cite{tarnawski2021piper} also finds a parallel strategy with both inter- and intra-op parallelism, but it relies on manually designed intra-op parallelism strategies and analyzes on a uniform network topology and asynchronous pipeline parallel schedules.

\noindent \textbf{Techniques for training large-scale models.}
In addition to parallelization, there are other complementary techniques for training large-scale models, such as memory optimization~\cite{chen2016training, jain2019checkmate,chen2021actnn,huang2020swapadvisor,kirisame2020dynamic,ren2021zero}, communication compression \cite{bai2021gradient,vogels2019powersgd}, and low-precision training\cite{micikevicius2017mixed}.
\sys can incorporate many of these techniques.
For example, \sys uses rematerialization to reduce memory usage and uses mixed-precision training to accelerate computation.

\noindent \textbf{Compilers for deep learning.}
Compiler techniques have been introduced to optimize the execution of DL models~\cite{team2016theano,xla2017google, chen2018tvm,zheng2020ansor,jia2019taso,ma2020rammer,wang2021pet}.
Most of them focus on optimizing the computation for a single device.
In contrast, \sys is a compiler that supports a comprehensive space of execution plans for distributed training.

\noindent \textbf{Distributed tensor computation in other domains.}
Besides deep learning, libraries and compilers for distributed tensor computation have been developed for linear algebra \cite{blackford1997scalapack} and stencil computations \cite{denniston2016distributed}.
Unlike \sys, they do not consider necessary parallelization techniques for DL.
\section{Conclusion}
We present \sys, a new architecture for automated model-parallel distributed training, built on top of a new view of machine learning parallelization approaches: intra- and inter-operator parallelisms.
\sys constructs a hierarchical space and uses a set of compilation passes to derive efficient parallel execution plans at each parallelism level. \sys orchestrates the parallel execution on distributed compute devices on two different granularities.
Coming up with an efficient parallelization plan for distributed model-parallel deep learning is historically a labor-intensive task, and we believe \sys will democratize distributed model-parallel learning and accelerate the adoption of emerging large deep learning models.

\section{Acknowledgement}
\label{sec:acknowledgement}
We would like to thank Shibo Wang, Yu Emma Wang, Jinliang Wei, Zhen Zhang, Siyuan Zhuang, anonymous reviewers, and our shepherd, Ken Birman, for their insightful feedback.
In addition to NSF CISE Expeditions Award CCF-1730628, this research is supported by gifts from Alibaba Group, Amazon Web Services, Ant Group, CapitalOne, Ericsson, Facebook, Futurewei, Google, Intel, Microsoft, Nvidia, Scotiabank, Splunk, and VMware.

\bibliographystyle{plain}
\bibliography{ref}

\begin{thebibliography}{10}

\bibitem{abadi2016tensorflow}
Mart{\'\i}n Abadi, Paul Barham, Jianmin Chen, Zhifeng Chen, Andy Davis, Jeffrey
  Dean, Matthieu Devin, Sanjay Ghemawat, Geoffrey Irving, Michael Isard, et~al.
\newblock Tensorflow: a system for large-scale machine learning.
\newblock In {\em 12th USENIX Symposium on Operating Systems Design and
  Implementation (OSDI 16)}, pages 265--283, 2016.

\bibitem{athlur2022varuna}
Sanjith Athlur, Nitika Saran, Muthian Sivathanu, Ramachandran Ramjee, and Nipun
  Kwatra.
\newblock Varuna: scalable, low-cost training of massive deep learning models.
\newblock In {\em Proceedings of the Seventeenth European Conference on
  Computer Systems}, pages 472--487, 2022.

\bibitem{aws}
{AWS Cluster Configuratoins}.
\newblock \url{https://aws.amazon.com/ec2/instance-types/p3/}.

\bibitem{aydin2016distributed}
Kevin Aydin, MohammadHossein Bateni, and Vahab Mirrokni.
\newblock Distributed balanced partitioning via linear embedding.
\newblock In {\em Proceedings of the Ninth ACM International Conference on Web
  Search and Data Mining}, pages 387--396, 2016.

\bibitem{ba2016layer}
Jimmy~Lei Ba, Jamie~Ryan Kiros, and Geoffrey~E Hinton.
\newblock Layer normalization.
\newblock {\em arXiv preprint arXiv:1607.06450}, 2016.

\bibitem{bai2021gradient}
Youhui Bai, Cheng Li, Quan Zhou, Jun Yi, Ping Gong, Feng Yan, Ruichuan Chen,
  and Yinlong Xu.
\newblock Gradient compression supercharged high-performance data parallel dnn
  training.
\newblock In {\em Proceedings of the ACM SIGOPS 28th Symposium on Operating
  Systems Principles CD-ROM}, pages 359--375, 2021.

\bibitem{barham2022pathways}
Paul Barham, Aakanksha Chowdhery, Jeff Dean, Sanjay Ghemawat, Steven Hand,
  Daniel Hurt, Michael Isard, Hyeontaek Lim, Ruoming Pang, Sudip Roy, et~al.
\newblock Pathways: Asynchronous distributed dataflow for ml.
\newblock {\em Proceedings of Machine Learning and Systems}, 4, 2022.

\bibitem{blackford1997scalapack}
L~Susan Blackford, Jaeyoung Choi, Andy Cleary, Eduardo D'Azevedo, James Demmel,
  Inderjit Dhillon, Jack Dongarra, Sven Hammarling, Greg Henry, Antoine
  Petitet, et~al.
\newblock {\em ScaLAPACK users' guide}.
\newblock SIAM, 1997.

\bibitem{jax2018github}
James Bradbury, Roy Frostig, Peter Hawkins, Matthew~James Johnson, Chris Leary,
  Dougal Maclaurin, George Necula, Adam Paszke, Jake Vander{P}las, Skye
  Wanderman-{M}ilne, and Qiao Zhang.
\newblock {JAX}: composable transformations of {P}ython+{N}um{P}y programs,
  2018.

\bibitem{brown2020language}
Tom~B Brown, Benjamin Mann, Nick Ryder, Melanie Subbiah, Jared Kaplan, Prafulla
  Dhariwal, Arvind Neelakantan, Pranav Shyam, Girish Sastry, Amanda Askell,
  et~al.
\newblock Language models are few-shot learners.
\newblock {\em arXiv preprint arXiv:2005.14165}, 2020.

\bibitem{cai2021tensoropt}
Zhenkun Cai, Xiao Yan, Kaihao Ma, Yidi Wu, Yuzhen Huang, James Cheng, Teng Su,
  and Fan Yu.
\newblock Tensoropt: Exploring the tradeoffs in distributed dnn training with
  auto-parallelism.
\newblock {\em IEEE Transactions on Parallel and Distributed Systems},
  33(8):1967--1981, 2021.

\bibitem{chen2021actnn}
Jianfei Chen, Lianmin Zheng, Zhewei Yao, Dequan Wang, Ion Stoica, Michael~W
  Mahoney, and Joseph~E Gonzalez.
\newblock Actnn: Reducing training memory footprint via 2-bit activation
  compressed training.
\newblock In {\em International Conference on Machine Learning}, 2021.

\bibitem{chen2018tvm}
Tianqi Chen, Thierry Moreau, Ziheng Jiang, Lianmin Zheng, Eddie Yan, Haichen
  Shen, Meghan Cowan, Leyuan Wang, Yuwei Hu, Luis Ceze, et~al.
\newblock Tvm: An automated end-to-end optimizing compiler for deep learning.
\newblock In {\em 13th USENIX Symposium on Operating Systems Design and
  Implementation (OSDI 18)}, pages 578--594, 2018.

\bibitem{chen2016training}
Tianqi Chen, Bing Xu, Chiyuan Zhang, and Carlos Guestrin.
\newblock Training deep nets with sublinear memory cost.
\newblock {\em arXiv preprint arXiv:1604.06174}, 2016.

\bibitem{denniston2016distributed}
Tyler Denniston, Shoaib Kamil, and Saman Amarasinghe.
\newblock Distributed halide.
\newblock {\em ACM SIGPLAN Notices}, 51(8):1--12, 2016.

\bibitem{du2021glam}
Nan Du, Yanping Huang, Andrew~M. Dai, Simon Tong, Dmitry Lepikhin, Yuanzhong
  Xu, Maxim Krikun, Yanqi Zhou, Adams~Wei Yu, Orhan Firat, Barret Zoph, Liam
  Fedus, Maarten Bosma, Zongwei Zhou, Tao Wang, Yu~Emma Wang, Kellie Webster,
  Marie Pellat, Kevin Robinson, Kathy Meier-Hellstern, Toju Duke, Lucas Dixon,
  Kun Zhang, Quoc~V Le, Yonghui Wu, Zhifeng Chen, and Claire Cui.
\newblock Glam: Efficient scaling of language models with mixture-of-experts,
  2021.

\bibitem{fan2021dapple}
Shiqing Fan, Yi~Rong, Chen Meng, Zongyan Cao, Siyu Wang, Zhen Zheng, Chuan Wu,
  Guoping Long, Jun Yang, Lixue Xia, et~al.
\newblock Dapple: A pipelined data parallel approach for training large models.
\newblock In {\em Proceedings of the 26th ACM SIGPLAN Symposium on Principles
  and Practice of Parallel Programming}, pages 431--445, 2021.

\bibitem{forrest2005cbc}
John Forrest and Robin Lougee-Heimer.
\newblock Cbc user guide.
\newblock In {\em Emerging theory, methods, and applications}, pages 257--277.
  INFORMS, 2005.

\bibitem{forrester2020computational}
Richard~J Forrester and Noah Hunt-Isaak.
\newblock Computational comparison of exact solution methods for 0-1 quadratic
  programs: Recommendations for practitioners.
\newblock {\em Journal of Applied Mathematics}, 2020, 2020.

\bibitem{gcp}
{Google Clould TPU Cluster Configurations}.
\newblock \url{https://cloud.google.com/tpu}.

\bibitem{huang2020swapadvisor}
Chien-Chin Huang, Gu~Jin, and Jinyang Li.
\newblock Swapadvisor: Pushing deep learning beyond the gpu memory limit via
  smart swapping.
\newblock In {\em Proceedings of the Twenty-Fifth International Conference on
  Architectural Support for Programming Languages and Operating Systems}, pages
  1341--1355, 2020.

\bibitem{huang2019gpipe}
Yanping Huang, Youlong Cheng, Ankur Bapna, Orhan Firat, Dehao Chen, Mia Chen,
  HyoukJoong Lee, Jiquan Ngiam, Quoc~V Le, Yonghui Wu, et~al.
\newblock Gpipe: Efficient training of giant neural networks using pipeline
  parallelism.
\newblock {\em Advances in neural information processing systems}, 32:103--112,
  2019.

\bibitem{jain2019checkmate}
Paras Jain, Ajay Jain, Aniruddha Nrusimha, Amir Gholami, Pieter Abbeel, Kurt
  Keutzer, Ion Stoica, and Joseph~E Gonzalez.
\newblock Checkmate: Breaking the memory wall with optimal tensor
  rematerialization.
\newblock {\em arXiv preprint arXiv:1910.02653}, 2019.

\bibitem{jia2019taso}
Zhihao Jia, Oded Padon, James Thomas, Todd Warszawski, Matei Zaharia, and Alex
  Aiken.
\newblock Taso: optimizing deep learning computation with automatic generation
  of graph substitutions.
\newblock In {\em Proceedings of the 27th ACM Symposium on Operating Systems
  Principles}, pages 47--62, 2019.

\bibitem{jia2018beyond}
Zhihao Jia, Matei Zaharia, and Alex Aiken.
\newblock Beyond data and model parallelism for deep neural networks.
\newblock {\em arXiv preprint arXiv:1807.05358}, 2018.

\bibitem{jiang2020unified}
Yimin Jiang, Yibo Zhu, Chang Lan, Bairen Yi, Yong Cui, and Chuanxiong Guo.
\newblock A unified architecture for accelerating distributed dnn training in
  heterogeneous gpu/cpu clusters.
\newblock In {\em 14th USENIX Symposium on Operating Systems Design and
  Implementation (OSDI 20)}, pages 463--479, 2020.

\bibitem{kim2019parallax}
Soojeong Kim, Gyeong-In Yu, Hojin Park, Sungwoo Cho, Eunji Jeong, Hyeonmin Ha,
  Sanha Lee, Joo~Seong Jeong, and Byung-Gon Chun.
\newblock Parallax: Sparsity-aware data parallel training of deep neural
  networks.
\newblock In {\em Proceedings of the Fourteenth EuroSys Conference 2019}, pages
  1--15, 2019.

\bibitem{kirisame2020dynamic}
Marisa Kirisame, Steven Lyubomirsky, Altan Haan, Jennifer Brennan, Mike He,
  Jared Roesch, Tianqi Chen, and Zachary Tatlock.
\newblock Dynamic tensor rematerialization.
\newblock {\em arXiv preprint arXiv:2006.09616}, 2020.

\bibitem{krizhevsky2014one}
Alex Krizhevsky.
\newblock One weird trick for parallelizing convolutional neural networks.
\newblock {\em arXiv preprint arXiv:1404.5997}, 2014.

\bibitem{lee2019automating}
Woo-Yeon Lee, Yunseong Lee, Joo~Seong Jeong, Gyeong-In Yu, Joo~Yeon Kim, Ho~Jin
  Park, Beomyeol Jeon, Wonwook Song, Gunhee Kim, Markus Weimer, et~al.
\newblock Automating system configuration of distributed machine learning.
\newblock In {\em 2019 IEEE 39th International Conference on Distributed
  Computing Systems (ICDCS)}, pages 2057--2067. IEEE, 2019.

\bibitem{lepikhin2020gshard}
Dmitry Lepikhin, HyoukJoong Lee, Yuanzhong Xu, Dehao Chen, Orhan Firat, Yanping
  Huang, Maxim Krikun, Noam Shazeer, and Zhifeng Chen.
\newblock Gshard: Scaling giant models with conditional computation and
  automatic sharding.
\newblock {\em arXiv preprint arXiv:2006.16668}, 2020.

\bibitem{li2020pytorchddp}
Shen Li, Yanli Zhao, Rohan Varma, Omkar Salpekar, Pieter Noordhuis, Teng Li,
  Adam Paszke, Jeff Smith, Brian Vaughan, Pritam Damania, et~al.
\newblock Pytorch distributed: Experiences on accelerating data parallel
  training.
\newblock {\em arXiv preprint arXiv:2006.15704}, 2020.

\bibitem{li2021terapipe}
Zhuohan Li, Siyuan Zhuang, Shiyuan Guo, Danyang Zhuo, Hao Zhang, Dawn Song, and
  Ion Stoica.
\newblock Terapipe: Token-level pipeline parallelism for training large-scale
  language models.
\newblock {\em arXiv preprint arXiv:2102.07988}, 2021.

\bibitem{ma2020rammer}
Lingxiao Ma, Zhiqiang Xie, Zhi Yang, Jilong Xue, Youshan Miao, Wei Cui,
  Wenxiang Hu, Fan Yang, Lintao Zhang, and Lidong Zhou.
\newblock Rammer: Enabling holistic deep learning compiler optimizations with
  rtasks.
\newblock In {\em 14th USENIX Symposium on Operating Systems Design and
  Implementation (OSDI 20)}, pages 881--897, 2020.

\bibitem{micikevicius2017mixed}
Paulius Micikevicius, Sharan Narang, Jonah Alben, Gregory Diamos, Erich Elsen,
  David Garcia, Boris Ginsburg, Michael Houston, Oleksii Kuchaiev, Ganesh
  Venkatesh, et~al.
\newblock Mixed precision training.
\newblock {\em arXiv preprint arXiv:1710.03740}, 2017.

\bibitem{mirhoseini2017device}
Azalia Mirhoseini, Hieu Pham, Quoc~V Le, Benoit Steiner, Rasmus Larsen, Yuefeng
  Zhou, Naveen Kumar, Mohammad Norouzi, Samy Bengio, and Jeff Dean.
\newblock Device placement optimization with reinforcement learning.
\newblock In {\em International Conference on Machine Learning}, pages
  2430--2439. PMLR, 2017.

\bibitem{moritz2018ray}
Philipp Moritz, Robert Nishihara, Stephanie Wang, Alexey Tumanov, Richard Liaw,
  Eric Liang, Melih Elibol, Zongheng Yang, William Paul, Michael~I Jordan,
  et~al.
\newblock Ray: A distributed framework for emerging ai applications.
\newblock In {\em 13th USENIX Symposium on Operating Systems Design and
  Implementation (OSDI 18)}, pages 561--577, 2018.

\bibitem{narayanan2019pipedream}
Deepak Narayanan, Aaron Harlap, Amar Phanishayee, Vivek Seshadri, Nikhil~R
  Devanur, Gregory~R Ganger, Phillip~B Gibbons, and Matei Zaharia.
\newblock Pipedream: generalized pipeline parallelism for dnn training.
\newblock In {\em Proceedings of the 27th ACM Symposium on Operating Systems
  Principles}, pages 1--15, 2019.

\bibitem{narayanan2021memory}
Deepak Narayanan, Amar Phanishayee, Kaiyu Shi, Xie Chen, and Matei Zaharia.
\newblock Memory-efficient pipeline-parallel dnn training.
\newblock In {\em International Conference on Machine Learning}, pages
  7937--7947. PMLR, 2021.

\bibitem{narayanan2021efficient}
Deepak Narayanan, Mohammad Shoeybi, Jared Casper, Patrick LeGresley, Mostofa
  Patwary, Vijay Korthikanti, Dmitri Vainbrand, Prethvi Kashinkunti, Julie
  Bernauer, Bryan Catanzaro, et~al.
\newblock Efficient large-scale language model training on gpu clusters using
  megatron-lm.
\newblock In {\em Proceedings of the International Conference for High
  Performance Computing, Networking, Storage and Analysis}, pages 1--15, 2021.

\bibitem{nccl}
NVIDIA.
\newblock The nvidia collective communication library, 2018.

\bibitem{paszke2019pytorch}
Adam Paszke, Sam Gross, Francisco Massa, Adam Lerer, James Bradbury, Gregory
  Chanan, Trevor Killeen, Zeming Lin, Natalia Gimelshein, Luca Antiga, et~al.
\newblock Pytorch: an imperative style, high-performance deep learning library.
\newblock In {\em Advances in Neural Information Processing Systems}, pages
  8024--8035, 2019.

\bibitem{peng2019generic}
Yanghua Peng, Yibo Zhu, Yangrui Chen, Yixin Bao, Bairen Yi, Chang Lan, Chuan
  Wu, and Chuanxiong Guo.
\newblock A generic communication scheduler for distributed dnn training
  acceleration.
\newblock In {\em Proceedings of the 27th ACM Symposium on Operating Systems
  Principles}, pages 16--29, 2019.

\bibitem{rajbhandari2020zero}
Samyam Rajbhandari, Jeff Rasley, Olatunji Ruwase, and Yuxiong He.
\newblock Zero: Memory optimizations toward training trillion parameter models.
\newblock In {\em SC20: International Conference for High Performance
  Computing, Networking, Storage and Analysis}, pages 1--16. IEEE, 2020.

\bibitem{rasley2020deepspeed}
Jeff Rasley, Samyam Rajbhandari, Olatunji Ruwase, and Yuxiong He.
\newblock Deepspeed: System optimizations enable training deep learning models
  with over 100 billion parameters.
\newblock In {\em Proceedings of the 26th ACM SIGKDD International Conference
  on Knowledge Discovery \& Data Mining}, pages 3505--3506, 2020.

\bibitem{ren2021zero}
Jie Ren, Samyam Rajbhandari, Reza~Yazdani Aminabadi, Olatunji Ruwase, Shuangyan
  Yang, Minjia Zhang, Dong Li, and Yuxiong He.
\newblock Zero-offload: Democratizing billion-scale model training.
\newblock {\em arXiv preprint arXiv:2101.06840}, 2021.

\bibitem{sergeev2018horovod}
Alexander Sergeev and Mike Del~Balso.
\newblock Horovod: fast and easy distributed deep learning in tensorflow.
\newblock {\em arXiv preprint arXiv:1802.05799}, 2018.

\bibitem{shazeer2018mesh}
Noam Shazeer, Youlong Cheng, Niki Parmar, Dustin Tran, Ashish Vaswani, Penporn
  Koanantakool, Peter Hawkins, HyoukJoong Lee, Mingsheng Hong, Cliff Young,
  et~al.
\newblock Mesh-tensorflow: Deep learning for supercomputers.
\newblock {\em arXiv preprint arXiv:1811.02084}, 2018.

\bibitem{shoeybi2019megatron}
Mohammad Shoeybi, Mostofa Patwary, Raul Puri, Patrick LeGresley, Jared Casper,
  and Bryan Catanzaro.
\newblock Megatron-lm: Training multi-billion parameter language models using
  model parallelism.
\newblock {\em arXiv preprint arXiv:1909.08053}, 2019.

\bibitem{tarnawski2021piper}
Jakub~M Tarnawski, Deepak Narayanan, and Amar Phanishayee.
\newblock Piper: Multidimensional planner for dnn parallelization.
\newblock {\em Advances in Neural Information Processing Systems}, 34, 2021.

\bibitem{xla2017google}
Google~XLA Team.
\newblock Xla: Optimizing compiler for machine learning, 2017.

\bibitem{team2016theano}
The Theano~Development Team, Rami Al-Rfou, Guillaume Alain, Amjad Almahairi,
  Christof Angermueller, Dzmitry Bahdanau, Nicolas Ballas, Fr{\'e}d{\'e}ric
  Bastien, Justin Bayer, Anatoly Belikov, et~al.
\newblock Theano: A python framework for fast computation of mathematical
  expressions.
\newblock {\em arXiv preprint arXiv:1605.02688}, 2016.

\bibitem{vogels2019powersgd}
Thijs Vogels, Sai~Praneeth Karinireddy, and Martin Jaggi.
\newblock Powersgd: Practical low-rank gradient compression for distributed
  optimization.
\newblock {\em Advances In Neural Information Processing Systems 32 (Nips
  2019)}, 32(CONF), 2019.

\bibitem{wang2021pet}
Haojie Wang, Jidong Zhai, Mingyu Gao, Zixuan Ma, Shizhi Tang, Liyan Zheng,
  Yuanzhi Li, Kaiyuan Rong, Yuanyong Chen, and Zhihao Jia.
\newblock Pet: Optimizing tensor programs with partially equivalent
  transformations and automated corrections.
\newblock In {\em 15th USENIX Symposium on Operating Systems Design and
  Implementation (OSDI 21)}, pages 37--54, 2021.

\bibitem{wang2019supporting}
Minjie Wang, Chien-chin Huang, and Jinyang Li.
\newblock Supporting very large models using automatic dataflow graph
  partitioning.
\newblock In {\em Proceedings of the Fourteenth EuroSys Conference 2019}, pages
  1--17, 2019.

\bibitem{xu2020automatic}
Yuanzhong Xu, HyoukJoong Lee, Dehao Chen, Hongjun Choi, Blake Hechtman, and
  Shibo Wang.
\newblock Automatic cross-replica sharding of weight update in data-parallel
  training.
\newblock {\em arXiv preprint arXiv:2004.13336}, 2020.

\bibitem{xu2021gspmd}
Yuanzhong Xu, HyoukJoong Lee, Dehao Chen, Blake Hechtman, Yanping Huang, Rahul
  Joshi, Maxim Krikun, Dmitry Lepikhin, Andy Ly, Marcello Maggioni, et~al.
\newblock Gspmd: General and scalable parallelization for ml computation
  graphs.
\newblock {\em arXiv preprint arXiv:2105.04663}, 2021.

\bibitem{yuan2021oneflow}
Jinhui Yuan, Xinqi Li, Cheng Cheng, Juncheng Liu, Ran Guo, Shenghang Cai, Chi
  Yao, Fei Yang, Xiaodong Yi, Chuan Wu, et~al.
\newblock Oneflow: Redesign the distributed deep learning framework from
  scratch.
\newblock {\em arXiv preprint arXiv:2110.15032}, 2021.

\bibitem{zagoruyko2016wide}
Sergey Zagoruyko and Nikos Komodakis.
\newblock Wide residual networks.
\newblock {\em arXiv preprint arXiv:1605.07146}, 2016.

\bibitem{zhang2020autosync}
Hao Zhang, Yuan Li, Zhijie Deng, Xiaodan Liang, Lawrence Carin, and Eric Xing.
\newblock Autosync: Learning to synchronize for data-parallel distributed deep
  learning.
\newblock {\em Advances in Neural Information Processing Systems}, 33, 2020.

\bibitem{zhang2022mics}
Zhen Zhang, Shuai Zheng, Yida Wang, Justin Chiu, George Karypis, Trishul
  Chilimbi, Mu~Li, and Xin Jin.
\newblock Mics: Near-linear scaling for training gigantic model on public
  cloud.
\newblock {\em arXiv preprint arXiv:2205.00119}, 2022.

\bibitem{zheng2020ansor}
Lianmin Zheng, Chengfan Jia, Minmin Sun, Zhao Wu, Cody~Hao Yu, Ameer Haj-Ali,
  Yida Wang, Jun Yang, Danyang Zhuo, Koushik Sen, et~al.
\newblock Ansor: Generating high-performance tensor programs for deep learning.
\newblock In {\em 14th USENIX Symposium on Operating Systems Design and
  Implementation (OSDI 20)}, pages 863--879, 2020.

\end{thebibliography}

\clearpage
\appendix
\newpage

\section{Proof of Submesh Shape Covering}
\label{sec:proof_submesh_shape}

We prove the following theorem which shows we can always find a solution that fully covers the cluster mesh $(N, M)$ with our selected submesh shapes in \S5.2:
(1) one-dimensional submeshes of shape $(1, 1), (1, 2), (1, 4) \ldots (1, 2^m)$ where $2^m = M$ and
(2) two-dimensional submeshes of shape $(2, M), (3, M), \ldots, (N, M)$ .

\begin{theorem}
	For a list of submesh shapes $(n_1, m_1), \ldots (n_S, m_S),$ if $ \sum_i n_i\cdot m_i = N\cdot M$ and each $(n_i, m_i)$ satisfies either (1) $n_i = 1$ and $m_i = 2^{p_i}$ is a power of 2 or (2) $m_i = M,$ then we can always cover the full $(N, M)$ mesh where $M = 2^m$ with these submesh shapes.
\end{theorem}
\begin{proof}
	We start with putting the second type submesh into the full mesh. In this case, because $m_i = M,$ these submeshes can cover the full second dimension of the full mesh. After putting all the second kind of submeshes into the mesh, we reduce the problem to fit a cluster mesh of shape $(N, M)$ with submeshes with shape $(1, 2^{p_1}), \ldots, (1, 2^{p_S})$ where all $p_i \in \{0, 1, \ldots, m-1\}$. Note that now we have
	\begin{equation}
		2^{p_1} + \cdots + 2^{p_S} = N \cdot 2^{m}.
		\label{eq:small-mesh}
	\end{equation}
	We start an induction on $m.$ When $m = 1,$ we have all $p_i = 0$ and thus all the submeshes are of shape $(1, 1),$ which means that all the submeshes can definitely cover the full mesh. Assume the above hold for all $m = 1, 2, \ldots, k-1$. When $m=k$, note that in this case the number of submeshes with $p_i = 0$ should be an even number, because otherwise the left hand side of Eq.~\ref{eq:small-mesh} will be an odd number while the right hand side is always an even number. Then we can split all submeshes with shape $p_i=0$ into pairs, and we co-locate each pair to form a $(1, 2)$ mesh. After this transformation, we have all $p_i > 0,$ so we can subtract all $p_i$ and $m$ by 1 and reduce to $m = k-1$ case. Therefore, the theorem holds by induction.
\end{proof}

\section{Model Specifications}
\label{sec:model_spec}

For GPT-3 models, we use sequence length = 1024 and vocabulary size = 51200 for all models.
Other parameters of the models are listed in Table.~\ref{table:gpt_spec}. The last column is the number of GPUs used to train the corresponding model.

For GShard MoE models, we use sequence length = 1024 and vocabulary size = 32000 for all models.
Other parameters of the models are listed in Table.~\ref{table:moe_spec}. The last column is the number of GPUs used to train the corresponding model.

For Wide-ResNet models, we use input image size = (224, 224, 3) and \#class = 1024 for all models.
Other parameters of the models are listed in Table.~\ref{table:wresnet_spec}. The last column is the number of GPUs used to train the corresponding model.

\begin{table}
	\centering
	\caption{GPT-3 Model Specification}
	\begin{tabular}{ccccc}
		\toprule
		\#params & Hidden size & \#layers & \#heads  & \#gpus \\
		\midrule
		350M & 1024 & 24 & 16 & 1  \\
		1.3B & 2048 & 24 & 32 & 4  \\
		2.6B & 2560 & 32 & 32 & 8  \\
		6.7B & 4096 & 32 & 32 & 16 \\
		15B  & 5120 & 48 & 32 & 32 \\
		39B  & 8192 & 48 & 64 & 64 \\
		\bottomrule
	\end{tabular}
	\label{table:gpt_spec}
\end{table}

\begin{table}
	\centering
	\caption{GShard MoE Model Specification}
	\scalebox{0.88}{
		\begin{tabular}{cccccc}
			\toprule
			\#params & Hidden size & \#layers & \#heads & \#experts & \#gpus \\
			\midrule
			380M & 768  & 8  & 16 & 8  & 1  \\
			1.3B & 768  & 16 & 16 & 16 & 4  \\
			2.4B & 1024 & 16 & 16 & 16 & 8  \\
			10B  & 1536 & 16 & 16 & 32 & 16 \\
			27B  & 2048 & 16 & 32 & 48 & 32 \\
			70B  & 2048 & 32 & 32 & 64 & 64 \\
			\bottomrule
		\end{tabular}
	}
	\label{table:moe_spec}
\end{table}

\begin{table}
	\centering
	\caption{Wide-ResNet Model Specification}
	\scalebox{0.92}{
		\begin{tabular}{ccccc}
			\toprule
			\#params & \#layers & Base channel & Width factor & \#gpus \\
			\midrule
			250M & 50 & 160 & 2  & 1  \\
			1B   & 50 & 320 & 2  & 4  \\
			2B   & 50 & 448 & 2  & 8  \\
			4B   & 50 & 640 & 2  & 16 \\
			6.8B & 50 & 320 & 16 & 32 \\
			13B  & 101 & 320 & 16 & 64 \\
			\bottomrule
		\end{tabular}
	}
	\label{table:wresnet_spec}
\end{table}

\section{Extra Case Study}
\label{sec:extra_case_study}
We visualize the parallelization strategies \sys finds for Wide-ResNet on 4 and 8 GPUs in Fig.~\ref{fig:evaluation:case_wresnet_appendix}.

\begin{figure*}
	\centering
	\begin{subfigure}[b]{\textwidth}
	\centering
    \includegraphics[width=.95\textwidth]{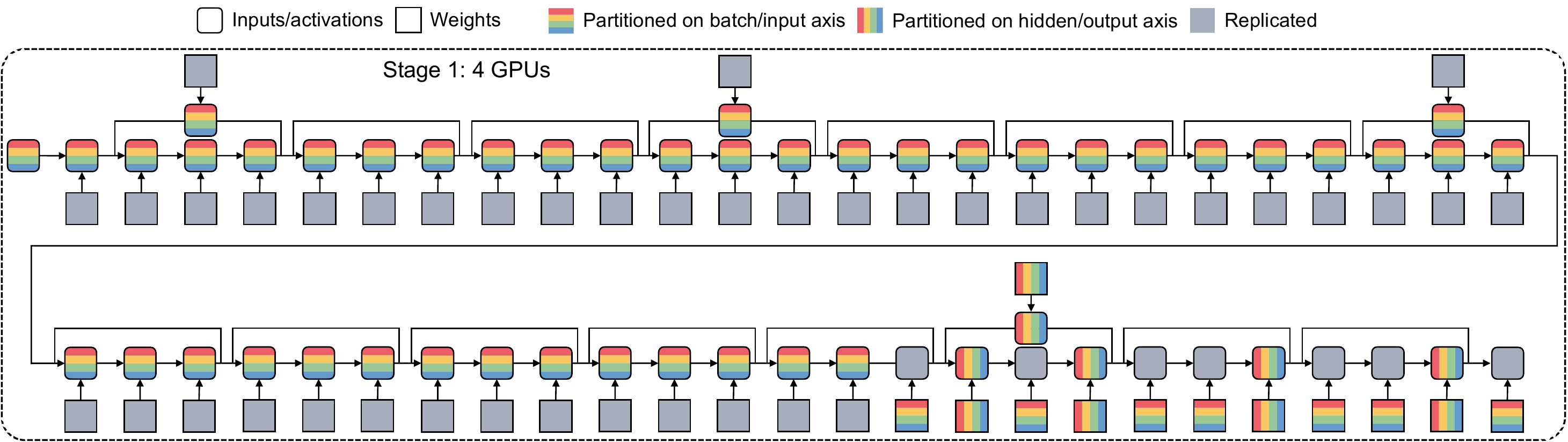}
    \caption{Parallel strategy of Wide-ResNet on 4 GPUs.}
    \end{subfigure}
    \begin{subfigure}[b]{\textwidth}
	\centering
    \vspace{2mm}
    \includegraphics[width=.95\textwidth]{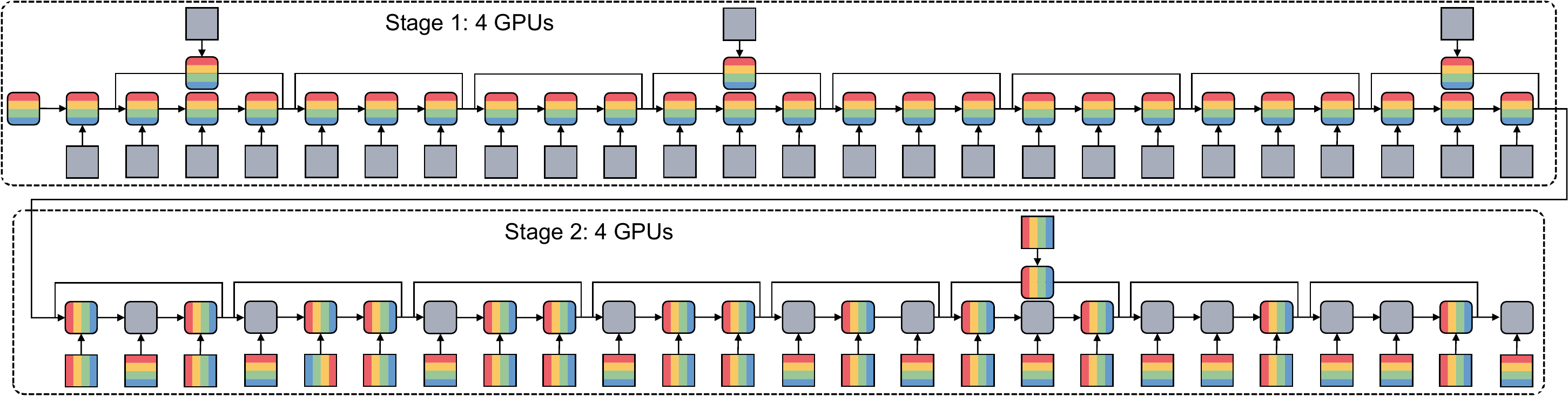}
    \caption{Parallel strategy of Wide-ResNet on 8 GPUs.}
    \end{subfigure}
	\caption{Visualization of the parallel strategy of Wide-ResNet on 4 and 8 GPUs. Different colors represent the devices a tensor is distributed on. Grey blocks indicate a tensor is replicated across all devices. The input data and resulting activation of each convolution or dense layer can be partitioned along the batch axis and the hidden axis. The weights can be partitioned along the input and output channel axis.}
	\label{fig:evaluation:case_wresnet_appendix}
\end{figure*}

\end{document}